\documentclass[11pt]{article}
\usepackage[margin=1in]{geometry}
\usepackage{amsmath,amssymb,amsthm,mathtools}
\usepackage{booktabs,graphicx,microtype,natbib}
\usepackage{setspace}

\usepackage{hyperref,xcolor,caption,subcaption}
\hypersetup{colorlinks=true,citecolor=blue!55!black,linkcolor=blue!55!black,urlcolor=blue!55!black,pdftitle={Deep Generalized Regression for Repeated Measurements},pdfauthor={Kexuan Li}}
\newtheorem{theorem}{Theorem}[section]
\newtheorem{proposition}[theorem]{Proposition}
\newtheorem{corollary}[theorem]{Corollary}
\newtheorem{lemma}[theorem]{Lemma}
\newtheorem{assumption}[theorem]{Assumption}
\theoremstyle{remark}

\newcommand{\E}{\mathbb E}
\newcommand{\PP}{\mathbb P}
\newcommand{\RR}{\mathbb R}
\newcommand{\cF}{\mathcal F}
\newcommand{\cO}{\mathcal O}
\newcommand{\norm}[1]{\lVert #1\rVert}
\newcommand{\ip}[2]{\langle #1,#2\rangle}
\newcommand{\Var}{\operatorname{Var}}
\newcommand{\Cov}{\operatorname{Cov}}
\newcommand{\tr}{\operatorname{tr}}
\newcommand{\KL}{\operatorname{KL}}
\newcommand{\TV}{\operatorname{TV}}

\newcommand{\op}{\operatorname}
\newcommand{\pnorm}[1]{\norm{#1}_{P_X}}
\title{Generalized Deep Regression for Repeated Measurements}
\author{Kexuan Li \\ Bristol Myers Squibb \\ \texttt{[kexuan.li.77@gmail.com]}}
\date{}

\begin{document}
\maketitle
\begin{abstract}
In this paper, we study the estimation of a marginal regression function from independent units with repeated binary, count, or continuous responses using ReLU deep neural networks. In the model, we assume that the dependence is generated by an unobserved random mean function within each unit. We then fit a neural network with a convex generalized regression loss. We show an oracle inequality by separating conditional measurement variation from between-unit variation. In addition, we prove that with $n$ units and $m$ measurements per unit, ReLU networks can attain an integrated mean squared error of order $n^{-1}+(nm)^{-2\beta/(2\beta+d)}$, up to logarithmic factors, over $\beta$-H\"older classes. We also derive a weighted oracle inequality for unequal cluster sizes and a rate for compositionally smooth functions. For pointwise ensemble inference, we give a projection central limit theorem and prove infinitesimal jackknife consistency under an explicit asymptotic linearity condition. Simulations and real data examples are provided to support our theoretical findings and practical implications.
\end{abstract}

\clearpage

\section{Introduction}

Repeated outcomes arise in longitudinal studies, digital health measurements, panel data, functional data, and experiments that record several responses from each independent unit \citep{diggle2002,baltagi2021,ramsaysilverman2005}. For example, a patient may be examined at a sequence of visits, a household may be surveyed over several years, and a wearable device may record many measurements from the same user, etc.. In each case, units can reasonably be treated as independent, whereas observations from the same unit remain dependent. Among these, a basic objective is to estimate the marginal mean of a response as a function of covariates. In practice, the relationship of the response and the covariates could be complex and nonlinear with interaction. At the same time, the response could be continuous, binary, or count data.

Classical methods address this within unit dependence in several ways. To name a few, linear and generalized linear mixed models represent heterogeneity through random effects, while generalized estimating equations target marginal regression parameters using a working correlation model \citep{lairdware1982,liangzeger1986}. Kernel and spline methods allow the marginal mean to vary nonparametrically and can incorporate within-subject correlation \citep{wang2003}. In functional data analysis, the effect of sampling frequency has been studied through minimax theory for discretely observed random functions. This work identified transitions between sparse and dense designs and showed that sufficiently dense sampling eventually leaves the number of independent curves as the limiting factor \citep{caiyuan2011,zhangwang2016}. Even these methods provide effective analyses for low dimensional covariates or specified basis structures, their use becomes more difficult when the true unknown regression function contains interactions, or unknown complex structure.

In recent years, deep neural networks have attracted considerable attention
because of their ability to represent complex structures and have been applied
in areas including computer vision and natural language processing
\citep{lecun2015deep}, drug discovery \citep{chen2018rise}, and environmental
and Earth system science \citep{reichstein2019deep}.  Besides its successful applications, there has also
been great progress on theoretical development of deep learning in statistical literature. Approximation results for ReLU networks showed that depth can exploit smoothness and compositional structure \citep{yarotsky2017}. \citet{schmidthieber2020} developed nonparametric risk bounds for sparse ReLU networks and showed that the estimator can adapt to hierarchical composition models under independent sampling. \citet{farrell2021} established estimation rates for deep empirical risk minimizers and used them to support inference on low-dimensional parameters. Related ideas have since been adapted to statistical models with specialized losses and dependence structures. Examples include partially linear Cox regression \citep{zhongmuellerwang2022}, partially linear quantile regression \citep{zhongwang2024}, semiparametric spatial regression \citep{lietal2023spatial}, and calibration of nonlinear ordinary differential equations from noisy observations \citep{lietal2024ode}, etc. This literature provides theory for deep estimators in a range of statistical settings, but generalized responses and repeated measurements have largely been studied along separate lines.

The present paper is motivated by two recent developments that address these lines separately. \citet{yanyaozhou2025} studied least-squares deep regression for repeated continuous measurements. Their analysis accommodates arbitrary sampling frequency, establishes the rate transition between sparse and dense designs, and develops results for several low-intrinsic-dimensional function classes. In a different direction, \citet{mengli2026} studied generalized nonparametric regression with independent observations. They derived error bounds for deep network estimators and proposed an ensemble subsampling method for pointwise inference on categorical and exponential-family means. These results leave open the marginal nonparametric problem in which generalized responses are repeatedly observed from the same units. 

To fill this gap, we consider the estimation of the regression function with generalized response using deep neural networks. Specifically, we assume that each unit is independent and the dependence is captured with unit-specific random mean functions. Conditional on these functions, the repeated responses follow a generalized mean model; the distribution of the random functions is otherwise left unspecified. The population target is the marginal mean function, estimated by minimizing a pooled convex generalized regression loss over a class of ReLU networks. The resulting oracle inequality separates a between-unit term from a conditional measurement term. For $n$ units with $m$ measurements per unit and a $\beta$-H\"older target on $[0,1]^d$, the integrated squared error is of order
\[
 n^{-1}+(nm)^{-2\beta/(2\beta+d)},
\]
up to logarithmic factors. Matching lower bounds for Bernoulli, Poisson, and Gaussian submodels show that the two components reflect distinct sources of information. We also give a weighted oracle inequality for unequal cluster sizes and rates for compositionally smooth functions. For pointwise inference, we study ensembles formed by subsampling complete units. A projection central limit theorem and consistency of the infinitesimal jackknife covariance estimator are established under stated bias and asymptotic-linearity conditions.

The remainder of the paper is organized as follows. Section~\ref{sec:model} introduces the repeated-measurement model and the marginal target. Section~\ref{sec:estimation} defines the network estimator. Section~\ref{sec:risk} gives the oracle inequalities, convergence rates, and minimax lower bounds. Section~\ref{sec:inference} develops unit-subsampling inference. Section~\ref{sec:simulation} presents the numerical study and a real data application. Proofs are collected in the appendix.

\paragraph{Notation.}
For two nonnegative sequences $a_n$ and $b_n$, write $a_n\lesssim b_n$ if $a_n\leq Cb_n$ for a constant $C$ independent of $n$, and write $a_n\asymp b_n$ when both $a_n\lesssim b_n$ and $b_n\lesssim a_n$ hold. The symbols $\E$, $\PP$, $\Var$, and $\Cov$ denote expectation, probability, variance, and covariance. For a measurable function $g$, let $\norm{g}_\infty$ be its supremum norm, $\pnorm{g}^2=\int g^2\,dP_X$, and $\ip{g}{h}=\int gh\,dP_X$. Constants denoted by $c$ or $C$ may change from line to line.

\section{Model and marginal target}\label{sec:model}
\subsection{Observations and within-unit dependence}
Suppose we are given $n$ independent units and each unit $i$ has $m_i$ measurements, indexed by $j=1,\ldots,m_i$. For the unit $i$, at measurement $j$, the covariate vector is $X_{ij}\in\mathcal X=[0,1]^d$, where $d$ is the covariate dimension, and the response is $Y_{ij}$. The response may be binary, a nonnegative count, or continuous. Write the complete record for unit $i$ as
\[
 \cO_i=\{(X_{ij},Y_{ij}):1\leq j\leq m_i\},\qquad i=1,\ldots,n.
\]
The cluster sizes $m_i\geq1$ are deterministic and may vary with $n$. The total number of observations is $N=\sum_{i=1}^n m_i$; in a balanced design, $m_i=m$ and $N=nm$.

Let $M$ be a jointly measurable random function on $\mathcal X$, and let $M_1,\ldots,M_n$ be independent copies of $M$. The value $M_i(x)$ is the conditional response mean for unit $i$ at covariate value $x$. Let $P_X$ denote the common probability distribution of the covariates on $\mathcal X$, and let $Q_{M_i,x}$ denote the conditional probability distribution of a response given $M_i$ and $X_{ij}=x$. Given $M_1,\ldots,M_n$, the pairs $(X_{ij},Y_{ij})$ are independent across all $(i,j)$ and satisfy
\begin{equation}\label{eq:model}
 X_{ij}\sim P_X,\qquad
 Y_{ij}\mid(X_{ij}=x,M_i)\sim Q_{M_i,x},\qquad
 \int y\,dQ_{M_i,x}(y)=M_i(x).
\end{equation}
The covariates are independent of the random functions. The same map $(M_i,x)\mapsto Q_{M_i,x}$ is used at every measurement and in every unit. Responses within a unit share $M_i$ and are therefore generally dependent after averaging over that random function. The distribution of $M$ describes heterogeneity between units.

Below, $(M,X,Y)$ denotes a generic draw from this mechanism: first draw $M$, then draw $X\sim P_X$ independently of $M$, and finally draw $Y\mid(M,X=x)\sim Q_{M,x}$. Expectations without a conditioning variable average over all the randomness in this draw, or in the observed sample when an estimator is involved.

\subsection{Marginal mean and loss}
Define the population marginal mean $\mu_0:\mathcal X\to\mathbb R$ and the unit-specific deviation $U_i:\mathcal X\to\mathbb R$ by
\[
 \mu_0(x)=\E M(x),\qquad U_i(x)=M_i(x)-\mu_0(x).
\]
Write $U=M-\mu_0$ for a generic copy of $U_i$. Thus $\E U(x)=0$ and $\E(Y_{ij}\mid X_{ij}=x)=\mu_0(x)$. The problem considered in this paper is to estimate the marginal mean
function $\mu_0$ from through deep neural networks.

We represent $\mu_0$ through an unknown real-valued regression function $f_0$ and a known function $b$, and define the loss $\ell$ by
\begin{equation}\label{eq:link}
 \mu_0(x)=b'\{f_0(x)\},\qquad \ell(y,t)=b(t)-yt,
\end{equation}
where $b'$ denotes the derivative of $b$, $y$ is a response value, and $t$ is a candidate value of the regression function on the link scale. The curvature condition below makes $b'$ strictly increasing on the relevant interval, so that $f_0=(b')^{-1}\circ\mu_0$ is well defined there. For logistic, Poisson, and squared-error losses, respectively,
\[
 b(t)=\log(1+e^t),\qquad b(t)=e^t,\qquad b(t)=t^2/2.
\]
These choices give $\mu_0(x)=\{1+\exp[-f_0(x)]\}^{-1}$, $\mu_0(x)=\exp\{f_0(x)\}$, and $\mu_0(x)=f_0(x)$, respectively. The estimation bounds use the conditional mean in \eqref{eq:model} and the moment assumptions below. The lower bounds use Bernoulli, Poisson, and Gaussian conditional distributions.

For a measurable real-valued function $v$ on $\mathcal X$, write $\pnorm{v}^2=\int_{\mathcal X}v(x)^2\,dP_X(x)$ and $\norm{v}_\infty=\sup_{x\in\mathcal X}|v(x)|$. The space $L_2(P_X)$ consists of functions with finite $\pnorm{v}$, with functions identified when they agree $P_X$-almost everywhere. Its inner product is $\ip{v}{w}=\int_{\mathcal X}v(x)w(x)\,dP_X(x)$.

\subsection{Regularity conditions and score variance}
The following assumption bounds the link-scale regression function, the variation of the unit-specific means, and the conditional response noise.

\begin{assumption}\label{ass:model}
There are constants $F>0$, $C_U<\infty$, $\sigma<\infty$, $K_\epsilon<\infty$, and $0<\kappa_- \leq \kappa_+<\infty$ such that
\[
 \norm{f_0}_\infty\leq F,\quad
 \kappa_-\leq b''(t)\leq\kappa_+\ \ (|t|\leq F),\quad
 \norm{U_i}_\infty\leq C_U\quad\hbox{almost surely}.
\]
The function $b$ is twice continuously differentiable on a neighborhood of $[-F,F]$. Define
\[
 \tau^2=\E\pnorm{U_i}^2.
\]
The centered errors $\epsilon_{ij}=Y_{ij}-M_i(X_{ij})$ satisfy
\begin{equation}\label{eq:bernstein}
 \E(|\epsilon_{ij}|^k\mid X_{ij},M_i)
 \leq \frac{k!}{2}\sigma^2 K_\epsilon^{\,k-2},\qquad k\geq2.
\end{equation}
\end{assumption}

Here $F$ bounds $f_0$, $C_U$ bounds the unit deviations, and $\kappa_-$ and $\kappa_+$ bound the curvature of the loss. The constants $\sigma^2$ and $K_\epsilon$ control the conditional moments of $\epsilon_{ij}$; in particular, $\sigma^2$ is an upper bound on its conditional variance. The quantity
\[
 \tau^2=\int_{\mathcal X}\Var\{M(x)\}\,dP_X(x)
\]
measures integrated between-unit variation. The conditional mean of $\epsilon_{ij}$ is zero. Bernoulli responses satisfy \eqref{eq:bernstein}, as do Poisson responses with uniformly bounded conditional means and Gaussian errors with fixed variance. The constants in Assumption~\ref{ass:model} are uniform as $n$ and the $m_i$ vary.

For a deterministic measurable candidate $f:\mathcal X\to[-F,F]$, define the excess population risk by
\[
 \mathcal R(f)=\E\{\ell(Y,f(X))-\ell(Y,f_0(X))\},
\]
where $(X,Y)$ has the generic marginal distribution defined above. For a random estimator $\widehat f$, $\mathcal R(\widehat f)$ and $\pnorm{\widehat f-f_0}^2$ evaluate the fitted function against this population distribution, and their outer expectations average over the training sample.

\begin{proposition}[Identification and curvature]\label{prop:curvature}
Under Assumption~\ref{ass:model}, for every measurable $f$ with $\norm{f}_\infty\leq F$,
\begin{equation}\label{eq:curvature}
 \frac{\kappa_-}{2}\pnorm{f-f_0}^2
 \leq\mathcal R(f)\leq
 \frac{\kappa_+}{2}\pnorm{f-f_0}^2,\qquad
 \pnorm{b'(f)-\mu_0}^2\leq\kappa_+^2\pnorm{f-f_0}^2.
\end{equation}
Thus $f_0$ is the unique population minimizer in this range, up to $P_X$-null sets.
\end{proposition}

Let $q_i\geq0$ be a deterministic weight assigned to unit $i$, with $\sum_{i=1}^nq_i=1$. Each observation within that unit receives weight $q_i/m_i$. Equal-unit weighting uses $q_i=1/n$, whereas equal-observation weighting uses $q_i=m_i/N$. For a fixed test function $w\in L_2(P_X)$, define the weighted residual score
\[
 S_{q,w}=\sum_{i=1}^n\frac{q_i}{m_i}\sum_{j=1}^{m_i}
       \{Y_{ij}-\mu_0(X_{ij})\}w(X_{ij}).
\]
\begin{proposition}[Score variance]\label{prop:score}
Under Assumption~\ref{ass:model},
\begin{equation}\label{eq:score}
 \Var(S_{q,w})=A_w\sum_iq_i^2+D_w\sum_i\frac{q_i^2}{m_i},
\end{equation}
where
\[
 A_w=\Var\{\ip{U}{w}\},\qquad
 D_w=\E\!\left[\Var\bigl(\{Y-\mu_0(X)\}w(X)\mid M\bigr)\right].
\]
Here $A_w$ is the variance of the unit-level conditional score mean, and $D_w$ is the average conditional variance of one measurement score. Both are finite under Assumption~\ref{ass:model}. For $m_i=m$ and $q_i=1/n$, \eqref{eq:score} becomes $A_w/n+D_w/(nm)$.
\end{proposition}

The design in \eqref{eq:model} describes random covariate locations independent of unit heterogeneity. Common scheduled visits, covariates shared across all measurements of a unit, and informative visit processes require corresponding changes to the design assumptions.

\section{Network estimator}\label{sec:estimation}
\subsection{The network class}
Let $\varrho(t)=\max\{t,0\}$ be the rectified linear unit (ReLU) activation, applied coordinatewise to vectors. A feedforward network with $L\geq1$ hidden layers has width vector
\[
 \mathbf p=(p_0,p_1,\ldots,p_L,p_{L+1}),\qquad p_0=d,\quad p_{L+1}=1.
\]
Thus $p_\ell$ is the number of neurons in hidden layer $\ell$ for $\ell=1,\ldots,L$. For each $\ell=1,\ldots,L+1$, let $W_\ell\in\mathbb R^{p_\ell\times p_{\ell-1}}$ be a weight matrix and $a_\ell\in\mathbb R^{p_\ell}$ a bias vector. Write $\theta$ for the vector formed by all entries of these matrices and vectors. The network is defined recursively by
\[
 \begin{aligned}
 z_0(x)&=x,\\
 z_\ell(x)&=\varrho\{W_\ell z_{\ell-1}(x)+a_\ell\},
       &&\ell=1,\ldots,L,\\
 f_\theta^{\rm raw}(x)&=W_{L+1}z_L(x)+a_{L+1}.
 \end{aligned}
\]
The output layer is scalar and affine. Define the clipping map
\[
 \operatorname{clip}_F(t)=\max\{-F,\min\{t,F\}\},\qquad
 f_\theta(x)=\operatorname{clip}_F\{f_\theta^{\rm raw}(x)\}.
\]
Here $F$ is the bound in Assumption~\ref{ass:model}. Clipping is a fixed output operation and introduces no trainable parameters.

For an integer sparsity budget $s\geq0$ and a parameter bound $B>0$, define
\[
 \cF(L,\mathbf p,s,B,F)
 =
 \left\{f_\theta:\ \|\theta\|_0\leq s,\ \|\theta\|_\infty\leq B\right\}.
\]
The quantity $\|\theta\|_0$ counts the nonzero entries of $\theta$, including both weights and biases, and $\|\theta\|_\infty$ is the largest absolute parameter value. Consequently, $L$ and $\mathbf p$ specify the architecture, $s$ restricts the number of active parameters, $B$ restricts their magnitudes, and $F$ bounds the fitted link-scale function. The zero pattern is allowed to vary within the class. A network with fewer active connections is represented by setting the remaining weights to zero.

\subsection{Weighted empirical risk minimization}
We use $\cF$ as shorthand for a chosen deterministic class $\cF(L,\mathbf p,s,B,F)$. Its architecture and parameter bounds may depend on the sample sizes. The oracle inequality in Section~\ref{sec:risk} also applies to other function classes satisfying its stated conditions.

For the unit weights $q_i$ defined in Section~\ref{sec:model}, define the weighted empirical loss by
\begin{equation}\label{eq:weightedloss}
 \widehat L_q(f)=\sum_i\frac{q_i}{m_i}\sum_j
     \{b(f(X_{ij}))-Y_{ij}f(X_{ij})\}.
\end{equation}
Let $\widehat f\in\cF$ be a measurable approximate minimizer satisfying
\begin{equation}\label{eq:erm}
 \widehat L_q(\widehat f)\leq\inf_{f\in\cF}\widehat L_q(f)+\Delta,
 \qquad \Delta\geq0,\quad \E\Delta<\infty.
\end{equation}
The nonnegative, possibly random variable $\Delta$ measures the optimization error relative to the smallest empirical loss in $\cF$; $\Delta=0$ corresponds to an exact minimizer. The marginal mean estimator is $\widehat\mu(x)=b'\{\widehat f(x)\}$. For a fixed finite architecture with bounded parameters, the parameter set is a finite union of compact sets, so the empirical minimum exists. We use a measurable minimizer or measurable approximate minimizer. The parameter representation also makes the suprema in the proofs measurable.

\subsection{Class complexity and weights}
For $\varepsilon>0$, let $N(\varepsilon,\cF,\|\cdot\|_\infty)$ be the smallest number of functions $f_1,\ldots,f_K\in\cF$ such that every $f\in\cF$ satisfies $\min_{1\leq k\leq K}\|f-f_k\|_\infty\leq\varepsilon$. This is the covering number of $\cF$ in the uniform norm. Its logarithm,
\[
 H(\varepsilon)=\log N(\varepsilon,\cF,\norm{\cdot}_\infty),
\]
is the metric entropy at resolution $\varepsilon$. The symbol $N$ with these three arguments denotes a covering number; the scalar $N=\sum_i m_i$ denotes the total sample size. All logarithms are natural.

Three summaries of the observation weights enter the risk bounds:

\begin{equation}\label{eq:weights}
 A_q=\sum_iq_i^2,\qquad
 V_q=\sum_i\frac{q_i^2}{m_i},\qquad
 w_q=\max_i\frac{q_i}{m_i}.
\end{equation}
The quantities $A_q$ and $V_q$ occur in the exact score variance. The quantity $w_q$ controls the linear term in the conditional Bernstein inequality.

\section{Estimation bounds}\label{sec:risk}
\subsection{Oracle inequality}
The bound below separates approximation error, between-unit variation, conditional measurement variation, and optimization error. Uniform separability means that $\cF$ has a countable dense subset in the uniform norm; the finite-architecture network classes above have this property.
\begin{theorem}[Weighted oracle inequality]\label{thm:oracle}
Suppose \eqref{eq:model} and Assumption~\ref{ass:model} hold, $\cF$ is nonempty, separable in the uniform norm, and uniformly bounded by $F$, and $H(\varepsilon)<\infty$. Then the estimator in \eqref{eq:erm} satisfies, for $0<\varepsilon\leq1$,
\begin{align}
 \E\pnorm{\widehat f-f_0}^2
 \leq C\bigg[&
 \inf_{f\in\cF}\pnorm{f-f_0}^2+\tau^2 A_q \notag\\
 &+(V_q+w_q)\{H(\varepsilon)+1\}
       +\varepsilon+\E\Delta\bigg].\label{eq:oracle}
\end{align}
The same bound holds for $\E\pnorm{\widehat\mu-\mu_0}^2$, with a different constant. The constants depend on $F,C_U,\sigma,K_\epsilon,b,\kappa_-$ and $\kappa_+$, but not on $n$, $m_i$, $q_i$, or the class complexity.
\end{theorem}

Theorem~\ref{thm:oracle} holds for a general covariate distribution $P_X$. The proof uses a finite net for the conditionally independent measurement process and a Hilbert-space bound for the unit process. The latter is
\[
 \E\left\|\sum_i q_i U_i\right\|_{P_X}^2=\tau^2 A_q.
\]
Thus the between-unit contribution is $\tau^2A_q$, while the class entropy enters the conditional measurement term. The infimum in \eqref{eq:oracle} is the squared approximation error for $f_0$ over $\cF$.

\begin{corollary}[Balanced measurements]\label{cor:balanced}
For $m_i=m$, $q_i=1/n$, and $N=nm$,
\begin{equation}\label{eq:balanced}
 \E\pnorm{\widehat\mu-\mu_0}^2
 \leq C\left[
 \inf_{f\in\cF}\pnorm{f-f_0}^2+
 \frac{\tau^2}{n}+\frac{H(\varepsilon)+1}{N}
 +\varepsilon+\E\Delta\right].
\end{equation}
\end{corollary}

\subsection{Unequal cluster sizes}
Theorem~\ref{thm:oracle} treats unequal deterministic $m_i$ without a comparability restriction. For equal-unit weights, define
\[
 m_H=\left(\frac1n\sum_i m_i^{-1}\right)^{-1},\qquad
 m_{\min}=\min_{1\leq i\leq n}m_i,\qquad m_{\max}=\max_{1\leq i\leq n}m_i.
\]
Here $m_H$ is the harmonic mean of the cluster sizes, and $m_{\min}$ and $m_{\max}$ are their minimum and maximum. Equal-observation weighting assigns $q_i=m_i/N$, with $N=\sum_i m_i$.

\begin{corollary}[Two weighting choices]\label{cor:weights}
The stochastic terms in \eqref{eq:oracle} take the following forms:
\begin{align}
 q_i=1/n:\quad&
 \frac{\tau^2}{n}+
 \left(\frac1{nm_H}+\frac1{nm_{\min}}\right)\{H(\varepsilon)+1\},
 \label{eq:equalunits}\\
 q_i=m_i/N:\quad&
 \tau^2\frac{\sum_i m_i^2}{N^2}
       +\frac{H(\varepsilon)+1}{N},
 \label{eq:equalobs}
\end{align}
up to fixed multiplicative constants.
\end{corollary}

The harmonic mean describes the measurement variance under equal-unit weighting. The maximum weight remains in our risk bound even for bounded responses, because the measurement locations are random. If $m_{\max}/m_{\min}$ is bounded, the two measurement terms in \eqref{eq:equalunits} are comparable. Under equal-observation weighting, the between-unit term can be larger than $\tau^2/n$. Under \eqref{eq:model}, both weighting choices have the same population mean target: every marginal measurement has the same law. Informative cluster sizes are not covered by this conclusion.

\subsection{ReLU rates}
Let $\beta>0$ be a smoothness index, $t\geq1$ a dimension, and $R>0$ a radius. On a specified bounded rectangle $D\subset\mathbb R^t$, write $\beta=k+\alpha$, where $k=\lceil\beta\rceil-1$ and $\alpha\in(0,1]$. For a multi-index $\nu=(\nu_1,\ldots,\nu_t)$ of nonnegative integers, let $|\nu|=\sum_{a=1}^t\nu_a$ and let $D^\nu g$ denote the corresponding partial derivative, with $D^0g=g$. The H\"older class $\mathcal H^\beta_t(R)$ consists of functions with continuous partial derivatives through order $k$ satisfying
\[
 \sup_{x\in D}|D^\nu g(x)|\leq R\quad (|\nu|\leq k),
 \qquad
 |D^\nu g(x)-D^\nu g(y)|\leq R\|x-y\|_\infty^\alpha
 \quad (|\nu|=k,\ x,y\in D).
\]
For vectors, $\|x-y\|_\infty=\max_a|x_a-y_a|$. When $k=0$, these conditions bound the function itself and its H\"older increments. The domain $D$ is understood from context; for $f_0$, it is $[0,1]^d$.

Let $q\geq0$ be a fixed integer specifying $q+1$ component maps; this composition index is distinct from the unit weights $q_i$. Consider functions of the form
\begin{equation}\label{eq:composition}
 f_0=g_q\circ\cdots\circ g_0.
\end{equation}
For $u=0,\ldots,q$, the map $g_u$ takes a fixed bounded rectangle in $\RR^{d_u}$ into the domain of the next map, with $d_0=d$ and $d_{q+1}=1$. Each coordinate function of $g_u$ depends on at most $t_u$ of its input coordinates, where $1\leq t_u\leq d_u$, and belongs to $\mathcal H^{\beta_u}_{t_u}(R)$ as a function of those coordinates. The index $\beta_u>0$ specifies its smoothness. The dimensions, radii, and number of maps are fixed as the sample size grows. Define the effective smoothness indices $\beta_u^*$ and the rate quantity $\phi_N$ by
\begin{equation}\label{eq:phin}
 \beta_u^*=\beta_u\prod_{v=u+1}^q(\beta_v\wedge1),\qquad
 \phi_N=\max_{0\leq u\leq q}N^{-2\beta_u^*/(2\beta_u^*+t_u)}.
\end{equation}

Here $a\wedge b=\min\{a,b\}$, and an empty product equals one, so $\beta_q^*=\beta_q$. Let $p_{\max}=\max_{0\leq\ell\leq L+1}p_\ell$ be the largest layer width. We use $a_N\lesssim b_N$ for an inequality up to a constant independent of the sample sizes, $a_N\gtrsim b_N$ for the reverse inequality, and $a_N\asymp b_N$ when both hold.

\begin{corollary}[Network rates]\label{cor:rates}
Under the balanced model, suppose $f_0$ satisfies \eqref{eq:composition}. For sufficiently large fixed constants, choose a network class containing the approximation constructed in Appendix~\ref{app:rates}, with
\[
 L\asymp\log N,\qquad
 \min_{1\leq u\leq L}p_u\gtrsim N\phi_N,\qquad
 p_{\max}\leq C N^{c},\qquad
 s\asymp N\phi_N\log N,\qquad B_0\leq B\leq C N^c.
\]
Here $N=nm$, $B_0>0$ is a fixed constant determined by the component domains, and $c,C>0$ in the architecture bounds are fixed constants. If $\E\Delta\lesssim\phi_N$, then
\begin{equation}\label{eq:rate}
 \E\pnorm{\widehat\mu-\mu_0}^2
 \leq C\{\tau^2/n+\phi_N(\log N)^3\}.
\end{equation}
In particular, for $f_0\in\mathcal H^\beta_d(R)$,
\begin{equation}\label{eq:holderrate}
 \E\pnorm{\widehat\mu-\mu_0}^2
 \leq C\{\tau^2/n+(nm)^{-2\beta/(2\beta+d)}[\log(nm)]^3\}.
\end{equation}
\end{corollary}

The approximation input is the scalar ReLU theorem of \citet[Theorem 5]{schmidthieber2020}. Appendix~\ref{app:rates} gives the componentwise construction, error propagation, and entropy calculation needed here. The rate for the general compositional class is an upper bound; the lower bound below concerns ordinary H\"older classes.

For fixed positive $\tau$, balancing the two polynomial terms in \eqref{eq:holderrate} gives $m\asymp n^{d/(2\beta)}$. Logarithmic factors modify the boundary. This comparison describes the worst-case sampling-frequency transition in the stated class. If $\tau=0$, the bound has the ordinary independent-measurement form.

\subsection{Lower bound}
For each response family, we now specify a collection of data-generating distributions over which the minimax risk is taken. Take $P_X$ uniform on $[0,1]^d$, fix $\beta,R,F>0$, and restrict $f_0$ to $\mathcal H^\beta_d(R)$ and $[-F,F]$. Conditional laws in \eqref{eq:model} are Bernoulli with means in $[c,1-c]$, Poisson with means in $[c,C]$, or Gaussian with variance one and means in $[-C_G,C_G]$, where $C_G>0$. Take $0<c<1/2$ for Bernoulli and $0<c<1<C$ for Poisson. Fix positive upper bounds on $\norm{U}_\infty$ and $\E\pnorm{U}^2$; the class includes arbitrarily small nonzero random intercepts. Denote the resulting family by $\mathfrak P_{\beta,d}$, separately for each response type. All its members satisfy Assumption~\ref{ass:model} with uniform constants. For $P\in\mathfrak P_{\beta,d}$, let $\mu_P(x)=\E_P M(x)$ be its marginal mean, and let $\E_P$ denote expectation under the resulting observed-data law. Here $P$ indexes the full data-generating mechanism, whereas $P_X$ denotes only the covariate distribution.

\begin{theorem}[H\"older lower bound]\label{thm:lower}
For each of these three families and all $n,m\geq1$, there is a constant $c_0>0$, depending only on the fixed class constants, such that
\begin{equation}\label{eq:lower}
 \inf_{\widetilde\mu}\sup_{P\in\mathfrak P_{\beta,d}}
 \E_P\pnorm{\widetilde\mu-\mu_P}^2
 \geq c_0\{n^{-1}+(nm)^{-2\beta/(2\beta+d)}\}.
\end{equation}
The infimum is over all measurable estimators $\widetilde\mu$ based on the $n$ units and their $m$ measurements per unit. Together with \eqref{eq:holderrate}, this establishes the H\"older minimax rate up to logarithms.
\end{theorem}

The proof uses disjoint smooth bumps in an independent-response submodel for the $nm$ term. For the $n$ term, a unit has one of two fixed conditional means. Only the mixing probability changes between alternatives. This keeps the support fixed and makes the information bound independent of $m$.

\section{Unit-subsampling inference}\label{sec:inference}
\subsection{Subsample fits and the complete ensemble}
Assume balanced measurements, $m_i=m=m_n$, so that the unit records $\cO_1,\ldots,\cO_n$ are independent and identically distributed for each $n$. Their common law may vary with $n$ through $m_n$. Fix $J\geq1$ evaluation points $\mathbf x=(x_1,\ldots,x_J)$ in $\mathcal X$, with $J$ independent of $n$. The target vector is
\[
 \mu_{\mathbf x}=(\mu_0(x_1),\ldots,\mu_0(x_J))^{\mathsf T}\in\mathbb R^J,
\]
where the superscript $\mathsf T$ denotes transpose.

Let $r=r_n$ be the number of complete units used in one subsample fit, with $1\leq r\leq n$. Write $T_r(\cO_1,\ldots,\cO_r;\omega)\in\mathbb R^J$ for its vector of fitted marginal means at $\mathbf x$. The seed $\omega$ represents training randomness and is independent of the records. The training rule is symmetric in distribution: permuting its $r$ unit arguments leaves its distribution over $\omega$ unchanged. An independent uniform random permutation before training provides this symmetry for an order-dependent algorithm.

Throughout this section, $\|\cdot\|$ denotes the Euclidean norm for vectors and $\|\cdot\|_{\rm op}$ the induced operator norm for matrices. For a positive definite matrix $A$, $A^{-1/2}$ is its symmetric inverse square root. Let $I_J$ be the $J\times J$ identity matrix and $N_J(0,I_J)$ the standard $J$-dimensional normal distribution. The symbols $\Longrightarrow$ and $\to_p$ denote convergence in distribution and in probability. All limits are taken as $n\to\infty$.

Average over training randomness while holding the records fixed to obtain the symmetric kernel
\[
 h_r(o_1,\ldots,o_r)=\E_\omega\{T_r(o_1,\ldots,o_r;\omega)\}.
\]
Its population mean is $\theta_r=\E\{h_r(\cO_1,\ldots,\cO_r)\}\in\mathbb R^J$. For a possible unit record $o$, define the first-order projection kernel and its covariance by
\[
 h_{1,r}(o)=\E\{h_r(o,\cO_2,\ldots,\cO_r)\}-\theta_r,\qquad
 \Gamma_r=\Cov\{h_{1,r}(\cO_1)\}.
\]
The expectation defining $h_{1,r}(o)$ is over $r-1$ independent unit records, with $o$ fixed. For $r=1$, this means $h_{1,1}(o)=h_1(o)-\theta_1$. In particular, $\E h_{1,r}(\cO_1)=0$.

For a subset $S\subset\{1,\ldots,n\}$, write $|S|$ for its cardinality and $\cO_S$ for its collection of unit records. The complete ensemble averages over all $\binom nr$ subsets of size $r$, and its first-order projection is
\begin{equation}\label{eq:ustat}
 U_{n,r}=\binom{n}{r}^{-1}\sum_{|S|=r}h_r(\cO_S),\qquad
 L_{n,r}=\frac rn\sum_{i=1}^n h_{1,r}(\cO_i).
\end{equation}
Define
\[
 R^{(2)}_{n,r}=U_{n,r}-\theta_r-L_{n,r},\qquad
 V_n=\Cov(L_{n,r})=\frac{r^2}{n}\Gamma_r.
\]
The vector $R^{(2)}_{n,r}$ collects the terms of order two and higher in the Hoeffding decomposition, and $V_n$ is the covariance of the first-order projection. This decomposition is standard for subsampling ensembles \citep{hoeffding1948,mentchhooker2016}.

\subsection{Projection limit and base-learner conditions}
\begin{theorem}[Projection central limit theorem]\label{thm:projection}
Suppose $\Gamma_r$ is positive definite and, for some $\delta>0$,
\begin{align}
 &n^{-\delta/2}\E\norm{\Gamma_r^{-1/2}h_{1,r}(\cO_1)}^{2+\delta}\longrightarrow0,
 \label{eq:lyapunov}\\
 &\E\norm{V_n^{-1/2}R^{(2)}_{n,r}}^2\longrightarrow0,\qquad
 \norm{V_n^{-1/2}(\theta_r-\mu_{\mathbf x})}\longrightarrow0.
 \label{eq:projconditions}
\end{align}
Then
\[
 V_n^{-1/2}(U_{n,r}-\mu_{\mathbf x})\ \Longrightarrow\ N_J(0,I_J).
\]
A sufficient condition for the first requirement in \eqref{eq:projconditions} is
\begin{equation}\label{eq:projection-sufficient}
 \frac{r-1}{r(n-1)}
 \E\norm{\Gamma_r^{-1/2}(h_r-\theta_r)}^2\longrightarrow0
\end{equation}
when $2\leq r\leq n$.
\end{theorem}

Condition~\eqref{eq:lyapunov} is a Lyapunov moment condition for the normalized first-order projection. The two conditions in \eqref{eq:projconditions} control the higher-order remainder and the pointwise bias $\theta_r-\mu_{\mathbf x}$ on the same covariance scale.

For covariance estimation, we use an asymptotic linear representation of the subsample fit. In the next assumption, $\psi_n(o)\in\mathbb R^J$ denotes the contribution of a single unit record $o$ to that representation, and $R_r\in\mathbb R^J$ is the remainder, including training randomness. Both may depend on $n$ through the sampling design, the subsample size, and the training rule.
\begin{assumption}[Asymptotic linearity at the evaluation points]\label{ass:al}
For a sequence $r=r_n\to\infty$ and a fixed constant $\rho\in(0,1)$ such that $r/n\leq\rho$,
\begin{equation}\label{eq:al}
 T_r-\theta_r=\frac1r\sum_{i=1}^r\psi_n(\cO_i)+R_r,
 \qquad \E\psi_n=0,\quad \E R_r=0.
\end{equation}
The single-unit covariance matrix $\Sigma_n=\Cov\{\psi_n(\cO_1)\}\in\mathbb R^{J\times J}$ is positive definite, and
\[
 r\E\norm{\Sigma_n^{-1/2}R_r}^2\longrightarrow0,
 \qquad \frac1n\E\norm{\Sigma_n^{-1/2}\psi_n(\cO_1)}^4\longrightarrow0.
\]
\end{assumption}
For inference about $\mu_{\mathbf x}$, also assume
\begin{equation}\label{eq:biasal}
 \sqrt n\,\norm{\Sigma_n^{-1/2}(\theta_r-\mu_{\mathbf x})}\longrightarrow0.
\end{equation}
The matrix $\Sigma_n$ may change in scale with $n$. For example, increasing eigenvalues allow the standard error in some directions to decrease more slowly than $n^{-1/2}$. Assumption~\ref{ass:al} specifies the stochastic behavior of the base learner, and \eqref{eq:biasal} requires its pointwise bias to be negligible relative to the standard error. These are additional conditions for inference.

\subsection{Finite ensembles and covariance estimation}
Let $B\geq2$ be the number of subsample fits in the ensemble. In this section $B$ denotes the ensemble size; the parameter-magnitude bound $B$ in Section~\ref{sec:estimation} is used only in specifying the network class. Conditional on the observed records, draw $S_1,\ldots,S_B$ independently and uniformly from the subsets of $\{1,\ldots,n\}$ of size $r<n$. Each subset contains distinct units, and different subsets may overlap or coincide. Use independent training seeds $\omega_1,\ldots,\omega_B$, also independent of the subsets and records. For $b=1,\ldots,B$ and $i=1,\ldots,n$, set
\[
 T_b=T_r(\cO_{S_b};\omega_b),\quad
 Z_{bi}=1(i\in S_b),\quad p=r/n,\quad \overline T=B^{-1}\sum_bT_b .
\]
Here $Z_{bi}$ is the indicator that unit $i$ belongs to subset $S_b$, $p=r/n$ is its inclusion probability, and $\overline T\in\mathbb R^J$ is the finite-ensemble prediction. The scalar $p$ is distinct from the network width vector $\mathbf p$.

Define $\widehat C_i\in\mathbb R^J$ as the empirical covariance between the inclusion indicator of unit $i$ and the fitted prediction. Let $a_n$ be the finite-population correction, $\widehat V_{\rm IJ}$ the unadjusted infinitesimal jackknife (IJ) covariance estimate, and $Q_{bi}\in\mathbb R^J$ the product used in its Monte Carlo correction. Specifically,
\begin{align}
 \widehat C_i&=B^{-1}\sum_b(Z_{bi}-p)(T_b-\overline T),&
 a_n&=\frac{n-1}{n-r},\label{eq:ij}\\
 \widehat V_{\rm IJ}&=a_n^2\sum_i\widehat C_i\widehat C_i^{\mathsf T},&
 Q_{bi}&=(Z_{bi}-p)(T_b-\overline T),\notag\\
 D_B&=\frac{a_n^2}{B}\sum_i\widehat{\Cov}_b(Q_{bi}),&
 \widehat V&=[\widehat V_{\rm IJ}-D_B]_+
                  +B^{-1}\widehat{\Cov}_b(T_b).\label{eq:correctedij}
\end{align}
For vectors $A_1,\ldots,A_B\in\mathbb R^J$, sample covariance in these formulas means
\[
 \widehat{\Cov}_b(A_b)
 =\frac1{B-1}\sum_{b=1}^B(A_b-\overline A)(A_b-\overline A)^{\mathsf T},
 \qquad \overline A=\frac1B\sum_{b=1}^B A_b.
\]
In $\widehat{\Cov}_b(Q_{bi})$, the unit index $i$ is held fixed. The matrices $D_B$, $\widehat V_{\rm IJ}$, and $\widehat V$ are all $J\times J$. The correction $D_B$ adjusts for finite-$B$ variation in estimating the inclusion covariances \citep{wager2014}. For a symmetric matrix $A$, $[A]_+$ replaces negative eigenvalues by zero and retains the eigenvectors. The final term $B^{-1}\widehat{\Cov}_b(T_b)$ estimates the conditional Monte Carlo covariance of the ensemble average.

\begin{theorem}[Infinitesimal jackknife and studentization]\label{thm:ij}
Suppose Assumption~\ref{ass:al} holds and the ensemble size $B=B_n$ satisfies $B/n\to\infty$. Then
\begin{equation}\label{eq:ijlimit}
 \begin{aligned}
 \norm{n\Sigma_n^{-1/2}\widehat V_{\rm IJ}\Sigma_n^{-1/2}-I_J}_{\rm op}&\to_p0,\\
 \norm{n\Sigma_n^{-1/2}\widehat V\Sigma_n^{-1/2}-I_J}_{\rm op}&\to_p0.
 \end{aligned}
\end{equation}
If \eqref{eq:biasal} holds, then for every nonzero fixed $v\in\RR^J$,
\[
 \frac{v^{\mathsf T}(\overline T-\mu_{\mathbf x})}
      {(v^{\mathsf T}\widehat Vv)^{1/2}}
 \Longrightarrow N(0,1),\qquad
 (\overline T-\mu_{\mathbf x})^{\mathsf T}\widehat V^{-1}
       (\overline T-\mu_{\mathbf x})\Longrightarrow\chi_J^2.
\]
\end{theorem}

Here $\chi_J^2$ denotes the chi-squared distribution with $J$ degrees of freedom. The vector $v$ specifies a fixed linear contrast of the $J$ marginal means. Covariance consistency in \eqref{eq:ijlimit} implies that $\widehat V$ is invertible with probability tending to one.

This sufficient regime makes Monte Carlo variation asymptotically negligible relative to the sampling covariance $\Sigma_n/n$. For any finite $B$, the exact covariance decomposition is
\begin{equation}\label{eq:totalvariance}
 \Cov(\overline T)=\Cov(U_{n,r})+
       B^{-1}\E\{\Cov_*(T_b)\},
\end{equation}
Here $\Cov_*(T_b)$ is the conditional covariance given $\cO_1,\ldots,\cO_n$, averaging over subset selection and training randomness, and the outer expectation averages over the records. Conditional on the records, $\E_*(T_b)=U_{n,r}$, so \eqref{eq:totalvariance} follows from the law of total covariance. Theorem~\ref{thm:ij} uses a regime in which the Monte Carlo contribution is asymptotically negligible. A regime in which that contribution is of leading order also requires a conditional limit theorem.

The pointwise bias condition \eqref{eq:biasal} concerns accuracy at the fixed evaluation points. To illustrate the distinction from integrated error, fix an interior point $x_0\in(0,1)^d$ and let $h>0$ tend to zero. Define
\[
 v_h(x)=\left(1-\norm{x-x_0}_\infty/h\right)_+
\]
where $(t)_+=\max\{t,0\}$. Then $v_h(x_0)=1$, whereas
\[
 \int_{[0,1]^d}v_h(x)^2\,dx\leq(2h)^d\longrightarrow0,
\]
with $dx$ denoting Lebesgue measure. This example explains why pointwise bias control is specified separately from the integrated risk bounds.

\section{Numerical study}\label{sec:simulation}
\subsection{Simulation design}
We generated $X_{ij}\sim{\rm Unif}([0,1]^3)$ independently and used
\[
 g(x)=-0.55+1.05\sin(\pi x_1x_2)+0.65(x_3-0.5)^2 .
\]
For binary outcomes, $\mu_0(x)=\op{expit}\{g(x)\}$; for counts,
$\mu_0(x)=\exp\{g(x)+0.35\}$. Let
\[
 a(x)=\frac{0.55+0.45\cos(\pi x_1)}
 {\{0.55^2+0.45^2/2\}^{1/2}},\qquad
 Z_i\sim {\rm Unif}(-\sqrt3,\sqrt3),\qquad
 M_i(x)=\mu_0(x)+\tau Z_i a(x).
\]
The normalization gives
$\E\int\{M_i(x)-\mu_0(x)\}^2dx=\tau^2$. We used $\tau=0$ and
$0.075$ for binary responses, and $\tau=0$ and $0.22$ for counts.
Conditional responses were Bernoulli or Poisson. The chosen values keep all
conditional means in their required ranges without clipping.

We used $n\in\{100,250,500\}$ and $m\in\{2,8,32\}$, with 50 replications per
configuration. The generalized network had two ReLU layers of widths 36 and 24
and was trained by Adam with early stopping. The main comparisons use an additive cubic-spline generalized linear model
and, for counts, the same network fitted by squared loss. The spline fit
provides an additive reference, while the squared-loss network allows us to
examine the choice of training loss within the same architecture.
Integrated squared error was evaluated on 5,000 new design
points. All tuning settings were fixed across replications.

We used two additional designs. To expose the between-unit floor, we generated
a one-dimensional Gaussian response with marginal mean
$\sin(2\pi x)+0.5\cos(4\pi x)$, $n=100$,
$m\in\{2,8,32,128\}$, and either no unit effect or a standard normal random
intercept. In the latter case $\tau^2/n=0.01$. For unequal cluster sizes,
$n=300$ and $m_i$ took values $2,4,8,16,32$ with probabilities
$0.28,0.24,0.20,0.16,0.12$. We compared equal-observation and equal-unit
training weights under $\tau=0$ and under count heterogeneity $\tau=0.28$.

A fourth design used an eight-dimensional nonadditive count regression. We set
$n=500$ and $m=16$, with covariates drawn uniformly from $[-1,1]^8$. The
marginal mean was
\[
\begin{split}
\mu_0(x)=\exp\{&-0.60+0.65\sin(\pi x_1x_2)
+0.50\cos\{\pi(x_3+x_4)/2\}\\
&+0.40\sin(\pi x_5x_6x_7)\}.
\end{split}
\]
The conditional Poisson mean was
$M_i(x)=\mu_0(x)(1+0.25Z_i)$ with $Z_i\sim{\rm Unif}(-1,1)$.
This design examines a regression surface containing pairwise and higher-order
interactions that are absent from the additive spline model.

\subsection{Simulation results}
Figures~\ref{fig:binary-risk} and \ref{fig:count-risk} show the balanced-design
results for the methods compared in the main text. Increasing $n$ or $m$ reduced error for the generalized network
throughout the grid. The effect of moderate unit heterogeneity is small at
these sample sizes because measurement error, approximation, and numerical
training error still contribute appreciably. At $n=250$ and $m=8$, the generalized network has lower mean error than
the additive spline fit for both response types. The two networks have
similar count errors; Table~\ref{tab:simulation} gives the numerical comparison.

\begin{figure}[p]
\centering
\includegraphics[width=\textwidth]{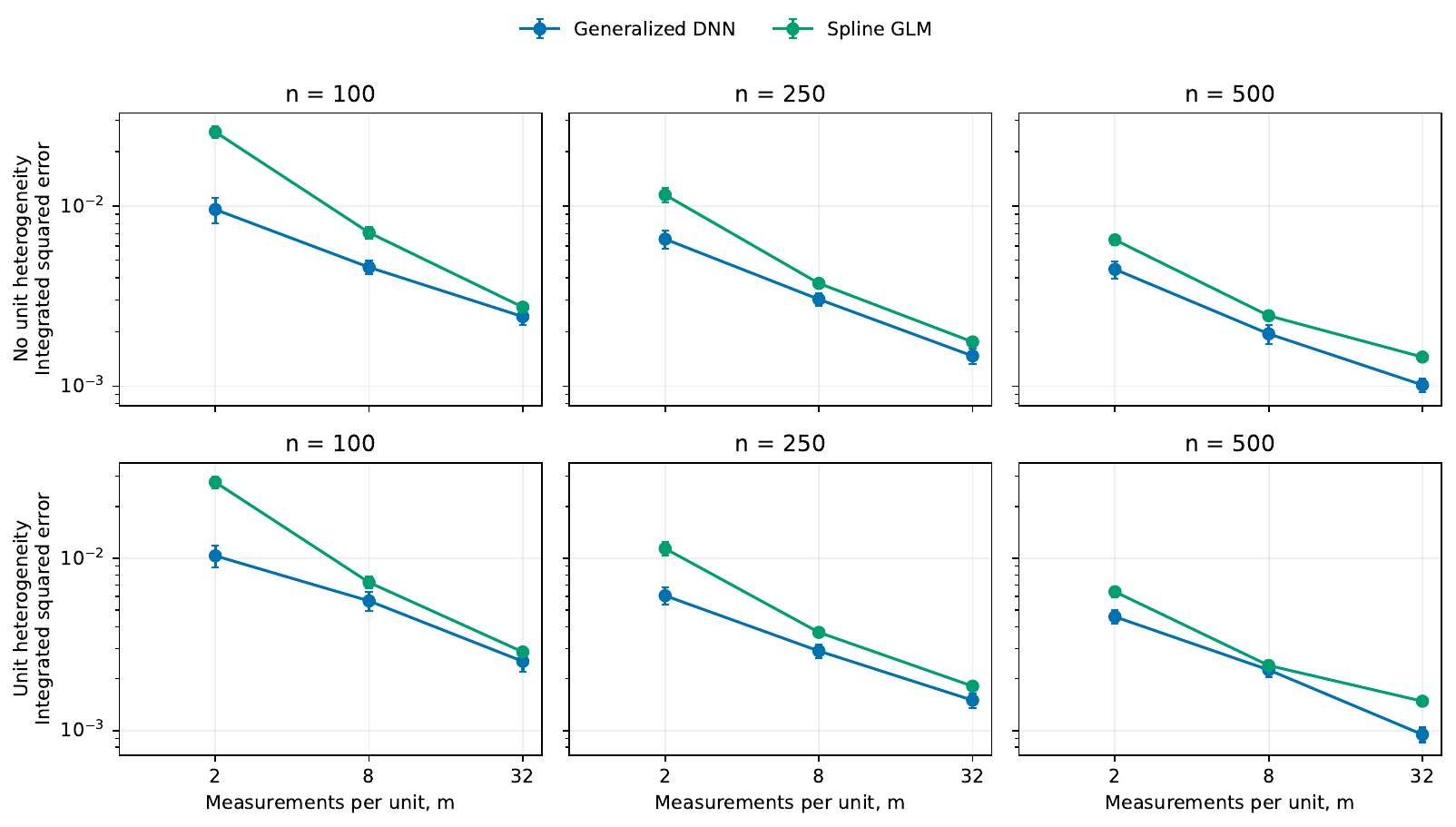}
\caption{Integrated squared error for binary responses. The upper row has
independent measurements and the lower row has unit heterogeneity
$\tau=0.075$. Error bars are 1.96 Monte Carlo standard errors over 50
replications.}\label{fig:binary-risk}
\end{figure}

\begin{figure}[p]
\centering
\includegraphics[width=\textwidth]{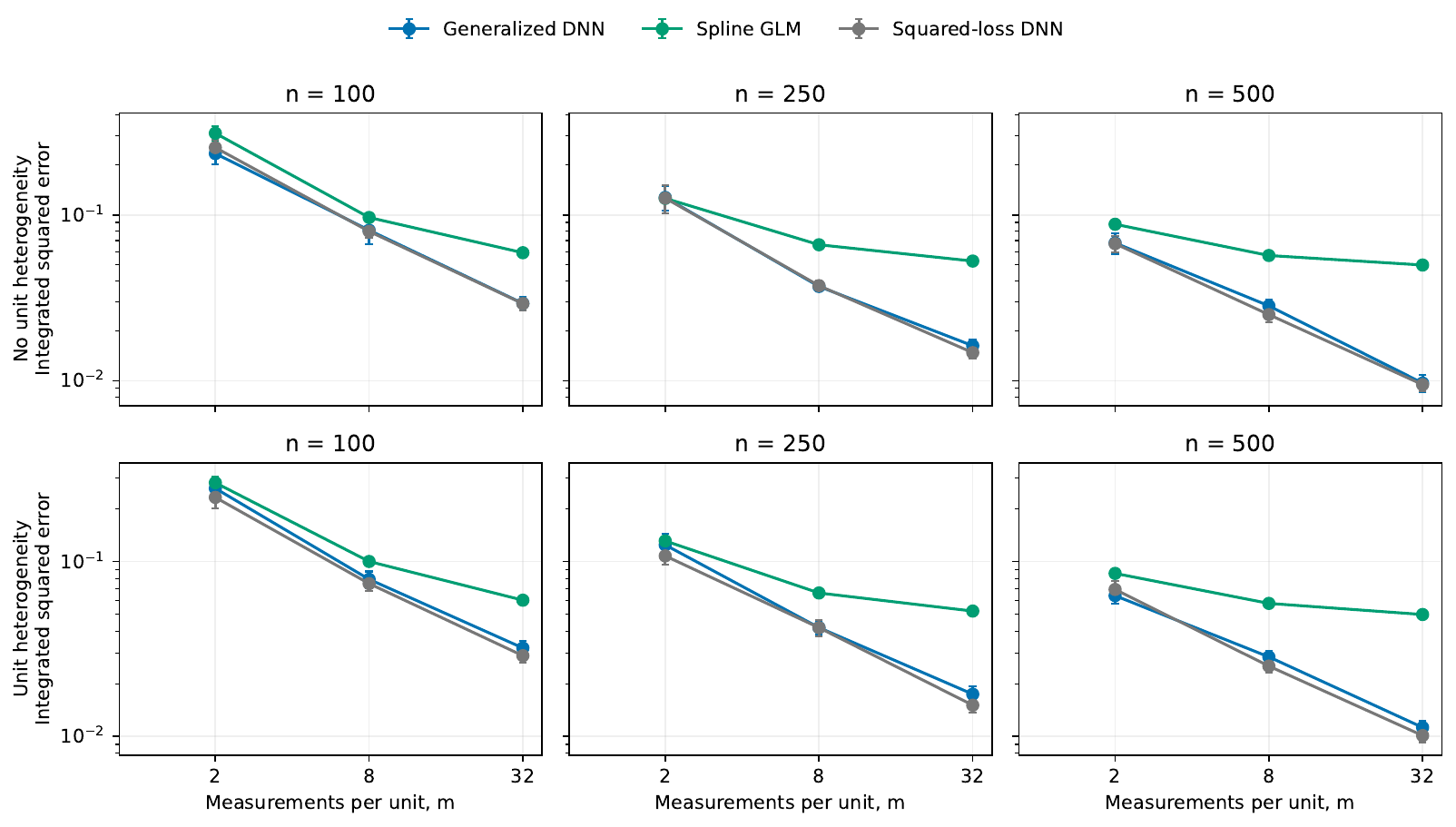}
\caption{Integrated squared error for count responses. The upper row has
independent measurements and the lower row has unit heterogeneity
$\tau=0.22$. Error bars are 1.96 Monte Carlo standard errors over 50
replications.}\label{fig:count-risk}
\end{figure}

\begin{table}[t]\centering
\caption{Integrated squared error at $n=250$, $m=8$, with moderate unit
heterogeneity. Parentheses contain Monte Carlo standard errors. Boldface
indicates the smallest displayed IMSE in each column.}
\label{tab:simulation}
\begin{tabular}{lcc}
\toprule
Method & Binary IMSE & Count IMSE \\
\midrule
Generalized DNN & \textbf{0.0029 (0.0001)} & \textbf{0.0419 (0.0019)} \\
Spline GLM & 0.0037 (0.0001) & 0.0661 (0.0008) \\
Squared-loss DNN & -- & \textbf{0.0419 (0.0023)} \\
\bottomrule
\end{tabular}
\end{table}

In the nonadditive count design, the generalized network has the smallest
mean error among the methods in the main comparison, followed by the squared-loss network;
Table~\ref{tab:nonadditive} gives the results. The difference between the two
networks is modest, while both improve on the additive spline fit. This design
supplements the three-dimensional experiment by placing nonlinear interactions
in the regression surface.

\begin{table}[t]\centering
\caption{Integrated squared error in the eight-dimensional nonadditive count
design. Monte Carlo standard errors are based on 50 replications. Boldface
indicates the smallest mean IMSE.}
\label{tab:nonadditive}
\begin{tabular}{lcc}
\toprule
Method & Mean IMSE & Monte Carlo SE \\
\midrule
Generalized DNN & \textbf{0.0476} & 0.0009 \\
Spline GLM & 0.1588 & 0.0004 \\
Squared-loss DNN & 0.0496 & 0.0007 \\
\bottomrule
\end{tabular}
\end{table}

The Gaussian experiment in Figure~\ref{fig:floor} isolates the two sampling
scales. Without a unit effect, the mean error falls from $0.116$ at $m=2$ to
$0.004$ at $m=128$. With a standard normal random intercept, it falls from
$0.210$ to $0.016$; most of the improvement occurs by $m=32$, after which the
error remains close to the reference value $\tau^2/n=0.01$. The remaining
difference includes finite-sample fitting error.

\begin{figure}[t]\centering
\includegraphics[width=0.66\textwidth]{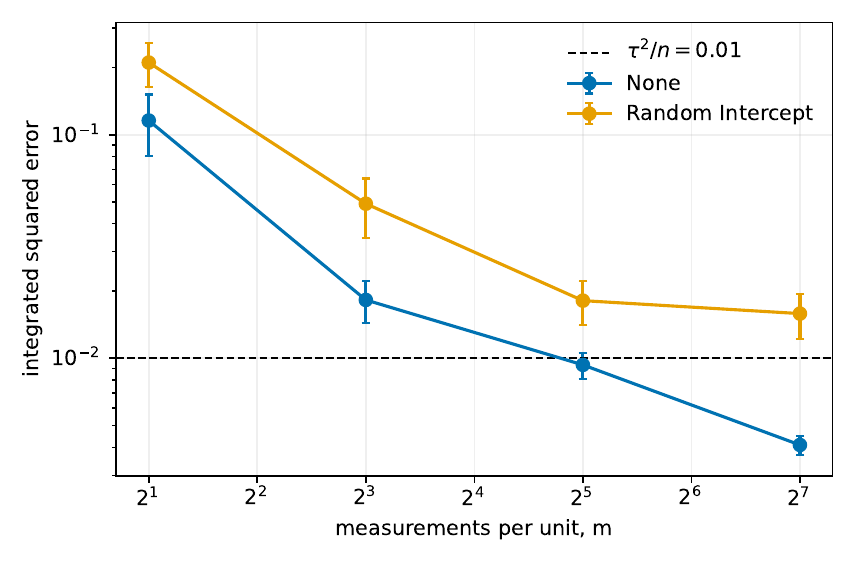}
\caption{Gaussian experiment separating measurement information from the
random-intercept floor. Error bars are 1.96 Monte Carlo standard errors over
50 replications.}\label{fig:floor}
\end{figure}

The unequal-size design averaged 2,877 observations. Observation weighting
assigned them to about 149 effective units, compared with 300 under equal-unit
weighting. Nevertheless, observation weighting has lower error in both
settings in Table~\ref{tab:unequal}; for this design, its smaller measurement
component and more stable numerical fit outweigh the larger between-unit
component. The oracle inequality does not impose a universal ordering between
the two choices.

\begin{table}[t]\centering
\caption{Integrated squared error with unequal cluster sizes. Boldface
indicates the smallest mean IMSE within each heterogeneity setting.}
\label{tab:unequal}
\begin{tabular}{llcc}
\toprule
Heterogeneity & Weighting & Mean IMSE & Monte Carlo SE \\
\midrule
None & Equal Observation & \textbf{0.0295} & 0.0011 \\
None & Equal Unit & 0.0502 & 0.0026 \\
Strong & Equal Observation & \textbf{0.0347} & 0.0015 \\
Strong & Equal Unit & 0.0596 & 0.0032 \\
\bottomrule
\end{tabular}
\end{table}

\subsection{Capital Bikeshare counts}
In this section, we apply the proposed model in a real data application. The Capital Bikeshare data contain 17,379 hourly rental counts from 731 days
in 2011 and 2012, together with weather and calendar variables
\citep{fanaeegama2014,ucibike}. We treated a day as an independent unit and
its hourly counts as repeated measurements. Predictors were hour, month, year,
day of week, season, holiday and working-day indicators, weather category,
temperature, humidity, and wind speed. Sine and cosine terms represented hour
and month. Five-fold cross-validation kept every hour of a day in the same
fold.

Table~\ref{tab:bike} reports Poisson deviance and root mean squared error for
the generalized network, the squared-loss network, and the additive spline
GLM. Among these methods, the generalized network has the smallest Poisson
deviance. The squared-loss network has a slightly smaller
RMSE, consistent with the difference between the two training criteria.
Figure~\ref{fig:bike} compares the observed and cross-validated mean hourly
profiles. The hourly grid and possible dependence across days differ from the sampling
assumptions in Section~\ref{sec:model}. This application illustrates predictive
performance under day-level cross-validation.

\begin{table}[t]\centering
\caption{Five-fold day-level cross-validation for hourly bicycle rentals.
Boldface indicates the smallest value in each column.}
\label{tab:bike}
\begin{tabular}{lcc}
\toprule
Method & Poisson deviance & RMSE \\
\midrule
Generalized DNN & \textbf{10.07} & 48.29 \\
Spline GLM & 53.16 & 104.70 \\
Squared-loss DNN & 20.40 & \textbf{46.67} \\
\bottomrule
\end{tabular}
\end{table}

\begin{figure}[t]\centering
\includegraphics[width=0.76\textwidth]{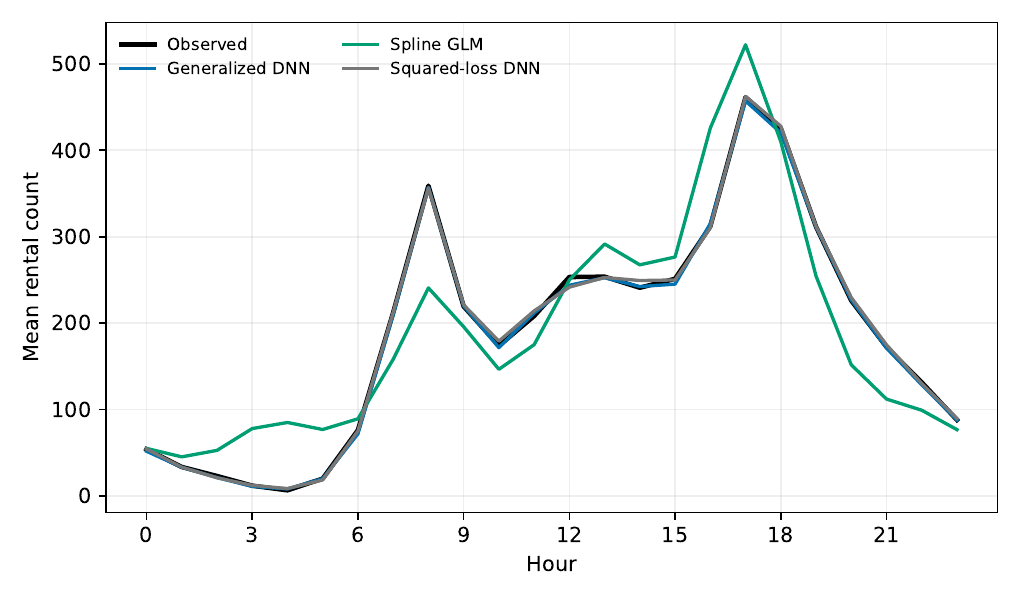}
\caption{Observed and cross-validated mean rental count by hour.}
\label{fig:bike}
\end{figure}

\clearpage

\bibliographystyle{apalike}
\bibliography{references}

\section*{Appendix}
\section{Identification and score variance}
\subsection{Conditional moment condition for the response families}
We verify the examples used in the paper. A centered Bernoulli variable has absolute value at most one, so \eqref{eq:bernstein} holds with $K_\epsilon=1$ and $\sigma^2=2$. For a Poisson variable of conditional mean $0\leq\lambda\leq C$, let $\epsilon=Y-\lambda$. Since $e^{|z|}\leq e^z+e^{-z}$,
\[
 \E(e^{|\epsilon|}\mid\lambda)
 \leq\exp\{\lambda(e-2)\}+\exp(\lambda/e)
 \leq A_C:=\exp\{C(e-2)\}+\exp(C/e).
\]
The inequality $|z|^k\leq k!e^{|z|}$ gives \eqref{eq:bernstein} with $K_\epsilon=1$ and $\sigma^2=2A_C$, uniformly in the conditional mean. For a centered Gaussian error of fixed variance $v$, the same argument gives $\E e^{|\epsilon|}\leq2e^{v/2}$, so one may take $K_\epsilon=1$ and $\sigma^2=4e^{v/2}$. These choices establish a common moment bound; their constants are not optimized.

\subsection{Proof of Proposition~\ref{prop:curvature}}
Conditional expectation in \eqref{eq:model} gives
\[
 \E(Y\mid X=x)=\E\{M(x)\}=\mu_0(x),
\]
because the design is independent of $M$. If $d(x)=f(x)-f_0(x)$, the integral remainder formula for Taylor expansion gives
\begin{align*}
 b\{f(x)\}-b\{f_0(x)\}-b'\{f_0(x)\}d(x)
 &=d(x)^2\int_0^1(1-t)b''\{f_0(x)+td(x)\}\,dt.
\end{align*}
The segment between $f_0(x)$ and $f(x)$ is in $[-F,F]$. Multiplying the curvature bounds by $1-t$ and integrating yields
\[
 \frac{\kappa_-}{2}d(x)^2
 \leq b\{f(x)\}-b\{f_0(x)\}-\mu_0(x)d(x)
 \leq\frac{\kappa_+}{2}d(x)^2.
\]
All terms are integrable: the functions are bounded and \eqref{eq:bernstein} bounds the conditional second moment of the error. Integrating over $P_X$ proves the first part of \eqref{eq:curvature}. Also,
\[
 |b'\{f(x)\}-b'\{f_0(x)\}|
 =\left|d(x)\int_0^1 b''\{f_0(x)+td(x)\}\,dt\right|
 \leq\kappa_+|d(x)|.
\]
Squaring and integrating proves the second part. Equality of population risks at their minimum implies $\pnorm{f-f_0}=0$ by the lower curvature bound. This proves the uniqueness statement. \qed

\subsection{Proof of Proposition~\ref{prop:score}}
Set $W_{ij}=\{Y_{ij}-\mu_0(X_{ij})\}w(X_{ij})$. The conditional mean of a single term is
\[
 \E(W_{ij}\mid M_i)
 =\int\{M_i(x)-\mu_0(x)\}w(x)\,dP_X(x)=\ip{U_i}{w}.
\]
Since $U_i$ is uniformly bounded and $\E(\epsilon_{ij}^2\mid X_{ij},M_i)\leq\sigma^2$,
\[
 \E W_{ij}^2
 =\E\{(U_i(X_{ij})^2+\epsilon_{ij}^2)w(X_{ij})^2\}
 \leq(C_U^2+\sigma^2)\pnorm{w}^2<\infty.
\]
The cross term vanishes by conditional centering. Conditional on $M_i$, the $W_{ij}$ are independent and identically distributed. The law of total variance therefore yields
\begin{align*}
 \Var\left(\frac1{m_i}\sum_jW_{ij}\right)
 &=\Var\{\ip{U_i}{w}\}
   +\E\left\{\frac1{m_i^2}\sum_j\Var(W_{ij}\mid M_i)\right\}\\
 &=A_w+\frac{D_w}{m_i}.
\end{align*}
The unit averages are independent across $i$. Multiplication by $q_i^2$ and summation proves \eqref{eq:score}. \qed

\section{Conditional empirical process and oracle inequality}\label{app:oracle}
We prove a finite-net bound explicitly. This avoids a localization step whose remainder would otherwise have to be checked separately.

Fix deterministic $g\in\cF$, let $d_f=f-g$, and write $\mathcal M=\sigma(M_1,\ldots,M_n)$. For one observation define
\[
 G_{ij}(f,g)=b\{f(X_{ij})\}-b\{g(X_{ij})\}-Y_{ij}d_f(X_{ij}).
\]
Conditional expectation given $\mathcal M$ will be written $\E_{\mathcal M}$. Put
\begin{equation}\label{eq:Z}
 Z_q(f,g)=\sum_{i,j}\frac{q_i}{m_i}
 \{G_{ij}(f,g)-\E_{\mathcal M}G_{ij}(f,g)\}.
\end{equation}

\begin{lemma}[Conditional moment bound]\label{lem:moments}
There are constants $v_0,c_0<\infty$, depending only on Assumption~\ref{ass:model}, such that for every deterministic $f,g\in\cF$, $k\geq2$, and every $(i,j)$,
\begin{equation}\label{eq:momentlemma}
 \E_{\mathcal M}|G_{ij}(f,g)-\E_{\mathcal M}G_{ij}(f,g)|^k
 \leq\frac{k!}{2}v_0\pnorm{f-g}^2c_0^{\,k-2}.
\end{equation}
\end{lemma}
\begin{proof}
Let $L_b=\sup_{|t|\leq F}|b'(t)|$ and $M_{\max}=L_b+C_U$, so that
$\norm{M_i}_\infty\leq M_{\max}$. For fixed $M_i$, put
\[
 a_i(x)=b\{f(x)\}-b\{g(x)\}-M_i(x)d_f(x),\qquad D=L_b+M_{\max}.
\]
The mean value theorem gives $|a_i(x)|\leq D|d_f(x)|\leq2FD$. Decompose the centered loss difference as
\[
 G_{ij}-\E_{\mathcal M}G_{ij}
 =A_{ij}-E_{ij},\quad
 A_{ij}=a_i(X_{ij})-\int a_i\,dP_X,\quad
 E_{ij}=\epsilon_{ij}d_f(X_{ij}).
\]
The error term has conditional mean zero. Furthermore,
\[
 |A_{ij}|\leq4FD,\qquad
 \E_{\mathcal M}A_{ij}^2\leq D^2\pnorm{d_f}^2,
\]
and hence
\[
 \E_{\mathcal M}|A_{ij}|^k
 \leq D^2(4FD)^{k-2}\pnorm{d_f}^2.
\]
For the error term, conditioning additionally on $X_{ij}$ and applying
\eqref{eq:bernstein} gives
\[
 \E_{\mathcal M}|E_{ij}|^k
 \leq\frac{k!}{2}\sigma^2K_\epsilon^{k-2}
       \int|d_f(x)|^k\,dP_X(x)
 \leq\frac{k!}{2}\sigma^2(2FK_\epsilon)^{k-2}\pnorm{d_f}^2.
\]
Finally use $|a-b|^k\leq2^{k-1}(|a|^k+|b|^k)$. For example, increasing
$v_0=8(D^2+\sigma^2+1)$ and
$c_0=\max(8FD,4FK_\epsilon,1)$ if necessary gives
\eqref{eq:momentlemma} for every $k\geq2$. No independence between $A_{ij}$ and
$E_{ij}$ is needed.
\end{proof}

\begin{lemma}[Finite-net bound]\label{lem:net}
Under the conditions of Theorem~\ref{thm:oracle}, for any fixed $\eta>0$ and
$g\in\cF$,
\begin{equation}\label{eq:offset}
 \E\sup_{f\in\cF}
 \bigl[|Z_q(f,g)|-\eta\pnorm{f-g}^2\bigr]_+
 \leq C_\eta(V_q+w_q)\{H(\varepsilon)+1\}+C_\eta\varepsilon .
\end{equation}
\end{lemma}
\begin{proof}
Write $w_{ij}=q_i/m_i$. Then $\sum_{i,j}w_{ij}=1$,
$\sum_{i,j}w_{ij}^2=V_q$, and $\max_{i,j}w_{ij}=w_q$.
For fixed $f,g$, the centered summands in \eqref{eq:Z} are conditionally independent given $\mathcal M$.
Expanding their exponential moments and using Lemma~\ref{lem:moments} gives, for
$|\lambda|<(c_0w_q)^{-1}$,
\begin{equation}\label{eq:mgf}
 \log\E_{\mathcal M}\exp\{\lambda Z_q(f,g)\}
 \leq \frac{\lambda^2v_0V_q\pnorm{f-g}^2}
                 {2(1-c_0w_q|\lambda|)} .
\end{equation}
Indeed, for a centered summand $\xi$ with weight $w$, its exponential series is at most
$1+\lambda^2w^2v_0\pnorm{f-g}^2/
\{2(1-c_0w|\lambda|)\}$. The inequality $\log(1+x)\leq x$, conditional independence, and $w\leq w_q$ give \eqref{eq:mgf}.

The exponential Markov inequality, optimized in $\lambda$ on each side, now gives
\begin{equation}\label{eq:tail}
 \PP_{\mathcal M}\left\{
 |Z_q(f,g)|>
 \sqrt{2v_0V_q\pnorm{f-g}^2t}+c_0w_qt
 \right\}\leq2e^{-t},\qquad t>0 .
\end{equation}
For example, with $v=v_0V_q\pnorm{f-g}^2$ and $c=c_0w_q$, substitute
$\lambda=\sqrt{2t/v}/(1+c\sqrt{2t/v})$ in \eqref{eq:mgf};
the case $v=0$ follows by continuity. Young's inequality bounds the threshold by
\[
 \eta\pnorm{f-g}^2+
 \left(\frac{v_0V_q}{2\eta}+c_0w_q\right)t
 \leq\eta\pnorm{f-g}^2+C_\eta(V_q+w_q)t.
\]

Take an $\varepsilon$-net $f_1,\ldots,f_K$ in $\cF$, with $\log K=H(\varepsilon)$.
The union bound in \eqref{eq:tail}, with $t=\log(2K)+u$, implies
\begin{align*}
 \PP_{\mathcal M}\bigg\{
 \max_{k\leq K}
 \bigl[|Z_q(f_k,g)|-\eta\pnorm{f_k-g}^2\bigr]_+
 >C_\eta(V_q+w_q)\{\log(2K)+u\}
 \bigg\}\leq e^{-u}.
\end{align*}
Integration in $u\geq0$ bounds the conditional expected maximum by
$C_\eta(V_q+w_q)\{\log(2K)+1\}$.

For an arbitrary $f$, choose $f_k$ with $\norm{f-f_k}_\infty\leq\varepsilon$.
The loss difference is Lipschitz in its fitted argument with random constant
$L_b+|Y_{ij}|$, so
\[
 |Z_q(f,g)-Z_q(f_k,g)|\leq\varepsilon\mathcal L,\qquad
 \mathcal L=\sum_{i,j}w_{ij}\{2L_b+|Y_{ij}|+\E_{\mathcal M}|Y_{ij}|\}.
\]
Since $|M_i(X_{ij})|\leq M_{\max}$ and
$\E(|\epsilon_{ij}|\mid\mathcal M)\leq\sigma$,
\[
 \E\mathcal L\leq2L_b+2(M_{\max}+\sigma).
\]
Also,
\[
 \left|\pnorm{f-g}^2-\pnorm{f_k-g}^2\right|
 \leq4F\varepsilon.
\]
It follows that the supremum over $\cF$ is bounded by the finite maximum plus
$\varepsilon(\mathcal L+4\eta F)$. Taking expectation proves
\eqref{eq:offset}. If a smallest net is not attained, a net with cardinality
arbitrarily close to the covering number gives the same bound after increasing the constant.
\end{proof}

\subsection{Proof of Theorem~\ref{thm:oracle}}
Fix deterministic $g\in\cF$ and let $\overline U_q=\sum_iq_iU_i$.
The conditional mean of the loss difference is
\[
 \E_{\mathcal M}G_{ij}(f,g)
 =\mathcal R(f)-\mathcal R(g)-\ip{U_i}{f-g}.
\]
Consequently the following identity holds for all $f\in\cF$ on the same sample:
\begin{equation}\label{eq:lossidentity}
 \widehat L_q(f)-\widehat L_q(g)
 =\mathcal R(f)-\mathcal R(g)-\ip{\overline U_q}{f-g}+Z_q(f,g).
\end{equation}
The uniform bound from Lemma~\ref{lem:net} therefore applies when $f=\widehat f$, although $\widehat f$ depends on the observations.

Let
\[
 T_\eta(g)=\sup_{f\in\cF}
   [|Z_q(f,g)|-\eta\pnorm{f-g}^2]_+.
\]
Combining \eqref{eq:erm} and \eqref{eq:lossidentity}, then applying Cauchy--Schwarz and Young's inequality, gives
\begin{align*}
 \mathcal R(\widehat f)
 &\leq\mathcal R(g)+|\ip{\overline U_q}{\widehat f-g}|
          +|Z_q(\widehat f,g)|+\Delta\\
 &\leq\mathcal R(g)+2\eta\pnorm{\widehat f-g}^2
       +(4\eta)^{-1}\pnorm{\overline U_q}^2+T_\eta(g)+\Delta .
\end{align*}
Set $u=\pnorm{\widehat f-f_0}^2$ and $a=\pnorm{g-f_0}^2$. Since
$\pnorm{\widehat f-g}^2\leq2u+2a$, Proposition~\ref{prop:curvature} yields
\begin{equation}\label{eq:absorb}
 (\kappa_-/2-4\eta)u
 \leq(\kappa_+/2+4\eta)a
       +(4\eta)^{-1}\pnorm{\overline U_q}^2+T_\eta(g)+\Delta .
\end{equation}
Choose $\eta=\kappa_-/16$; the coefficient of $u$ is $\kappa_-/4$.

The random functions are centered and independent as $L_2(P_X)$ elements.
Fubini's theorem is applicable because they are bounded. Thus
\begin{align}
 \E\pnorm{\overline U_q}^2
 &=\sum_iq_i^2\E\pnorm{U_i}^2+
       \sum_{i\ne k}q_iq_k\int\E\{U_i(x)U_k(x)\}\,dP_X(x)\notag\\
 &=\tau^2 A_q.\label{eq:hilbert}
\end{align}
Take expectations in \eqref{eq:absorb}, apply Lemma~\ref{lem:net}, and divide by $\kappa_-/4$.
This gives \eqref{eq:oracle} with the comparator $g$. Taking the infimum over deterministic comparators proves the stated inequality, by using a sequence approaching the infimum if necessary. The mean-scale bound follows from the second inequality in \eqref{eq:curvature}. \qed

\subsection{Proofs of Corollaries~\ref{cor:balanced} and \ref{cor:weights}}
For balanced data, $A_q=1/n$ and $V_q=w_q=1/(nm)$. Substitution into \eqref{eq:oracle} proves Corollary~\ref{cor:balanced}.
For $q_i=1/n$,
\[
 A_q=\frac1n,\qquad V_q=\frac1{n^2}\sum_i\frac1{m_i}=\frac1{nm_H},
 \qquad w_q=\frac1{nm_{\min}}.
\]
For $q_i=m_i/N$,
\[
 A_q=\frac{\sum_i m_i^2}{N^2},\qquad
 V_q=\sum_i\frac{m_i}{N^2}=\frac1N,\qquad w_q=\frac1N.
\]
These identities prove Corollary~\ref{cor:weights}. Finally,
$\E\widehat L_q(f)=\E\ell(Y,f(X))$ for every choice of these deterministic weights, since the weights sum to one and the marginal law is common. This verifies the population-target statement. \qed

\section{Network entropy and approximation}\label{app:rates}
\subsection{An entropy calculation}
\begin{lemma}\label{lem:entropy}
Let $W=\max_{0\leq u\leq L+1}p_u$, $B\geq1$, and $1\leq s\leq P$, where
$P=\sum_{u=1}^{L+1}p_u(p_{u-1}+1)$ is the number of available parameters.
For $\cF=\cF(L,\mathbf p,s,B,F)$ and $0<\varepsilon\leq1$,
\begin{equation}\label{eq:entropy}
 H(\varepsilon)\leq
 C(s+1)(L+1)\log\left\{\frac{C(B+1)(L+1)(W+1)}{\varepsilon}\right\}.
\end{equation}
\end{lemma}
\begin{proof}
First fix a support of at most $s$ nonzero parameters. Compare two parameter vectors with coordinatewise difference at most $\delta$ and magnitudes at most $B$.
Put $R=(B+1)(W+1)$. For input coordinates bounded by one, the maximum absolute activation after layer $u$ is at most $R^u$:
the recursion is $a_u\leq BW a_{u-1}+B$, with $a_0\leq1$, and $R$ bounds the right side relative to $R^{u-1}$.

Let $e_u$ be the maximum coordinate difference at layer $u$.
ReLU is 1-Lipschitz, so the same estimate, also valid for the affine output layer, gives
\[
 e_u\leq BW e_{u-1}+(W+1)R^{u-1}\delta,\qquad e_0=0.
\]
Induction yields
$e_u\leq u(W+1)R^{u-1}\delta$. For $u=L+1$ the output difference is at most
\[
 D_L\delta,\qquad D_L=(L+1)(W+1)R^L.
\]
Clipping cannot increase this difference. A grid of mesh at most
$\varepsilon/D_L$ for each active parameter therefore gives a uniform
$\varepsilon$-net of size at most
$(1+2BD_L/\varepsilon)^s$, after an immaterial fixed enlargement of the grid.

There are at most $\sum_{k=0}^s\binom{P}{k}\leq(s+1)(P+1)^s$ possible supports.
Taking the union of the grids over these supports gives
\[
 H(\varepsilon)\leq\log(s+1)+s\log(P+1)
          +s\log(1+2BD_L/\varepsilon).
\]
Since $P\leq(L+1)W(W+1)$, substitution of $D_L$ and $R$ proves
\eqref{eq:entropy}. The centers are themselves networks in the same bounded parameter class.
\end{proof}

\subsection{Approximation of the compositional class}
We state the approximation input in the form used here.
For a scalar $t$-variate function in a fixed H\"older ball of smoothness $\beta$,
\citet[Theorem 5]{schmidthieber2020} constructs, after a fixed affine change of domain if needed, a ReLU network with
\[
 \text{depth }O(k),\qquad
 \text{width }O(J),\qquad
 \text{sparsity }O(Jk),\qquad
 \norm{g-\widetilde g}_\infty
       \leq C\{J2^{-k}+J^{-\beta/t}\}.
\]
The parameter bound is fixed on fixed domains. This scalar approximation theorem is the external input; the following choices and composition argument give the form required for Corollary~\ref{cor:rates}.

Set $J_N=\lceil N\phi_N\rceil$ and $k_N=\lceil c_1\log_2 N\rceil$, with $c_1$ a sufficiently large fixed constant.
Each coordinate of $g_u$ depends on at most $t_u$ arguments, so it can be approximated by a scalar network of the displayed form.
The approximation may be clipped to its component output interval: the true output is in that interval, so clipping decreases error.
It can be realized by a fixed number of extra ReLU operations.
Parallelizing the finitely many component coordinates and concatenating the component maps gives depth $O(\log N)$, width $O(J_N)$, and sparsity $O(J_N\log N)$.
Affine changes between the fixed component rectangles require only fixed parameter bounds.

Write $\alpha_u=\beta_u\wedge1$ and
$\delta_u=\norm{\widetilde g_u-g_u}_\infty$, using the coordinatewise maximum norm.
A H\"older-smooth component is $\alpha_u$-H\"older as a map on its rectangle.
If $e_u$ is the error after composing through component $u$, then
\[
 e_u\leq C e_{u-1}^{\alpha_u}+\delta_u,\qquad e_{-1}=0.
\]
Repeated use of $(a+b)^\alpha\leq a^\alpha+b^\alpha$ for $0<\alpha\leq1$ gives
\begin{equation}\label{eq:propagate}
 \norm{\widetilde f-f_0}_\infty
 \leq C\sum_{u=0}^q\delta_u^{\,\prod_{v=u+1}^q\alpha_v}.
\end{equation}
Intermediate clipping ensures that all evaluations remain in the rectangles on which the regularity bounds hold.

For each $u$,
\[
 J_N\geq N^{\,t_u/(2\beta_u^*+t_u)}.
\]
Consequently,
\[
 J_N^{-\beta_u^*/t_u}
 \leq N^{-\beta_u^*/(2\beta_u^*+t_u)}
 \leq\phi_N^{1/2}.
\]
Since all smoothness indices are fixed and positive, $c_1$ can be chosen so that the terms from $J_N2^{-k_N}$ in \eqref{eq:propagate} are also bounded by $C\phi_N^{1/2}$. It follows that
\begin{equation}\label{eq:approx}
 \norm{\widetilde f-f_0}_\infty^2\leq C\phi_N.
\end{equation}
Clipping the final output to $[-F,F]$ preserves this bound.
The architecture constants in Corollary~\ref{cor:rates} are chosen to contain this construction; unused nodes and connections have zero weights.
If depth needs padding, the finitely many active component outputs can be passed through identity channels.
This only changes the sparsity by $O(\log N)$.

\subsection{Proof of Corollary~\ref{cor:rates}}
Choose $\varepsilon=N^{-2}$, for $N\geq2$. Under the architecture restrictions, the logarithm in \eqref{eq:entropy} is $O(\log N)$, so
\[
 H(N^{-2})\lesssim sL\log N
          \lesssim N\phi_N(\log N)^3.
\]
Equation~\eqref{eq:approx} bounds the squared $L_2(P_X)$ approximation error by $C\phi_N$, for any $P_X$. Substitution into \eqref{eq:balanced} gives
\[
 \E\pnorm{\widehat\mu-\mu_0}^2
 \lesssim \tau^2/n+\phi_N+\phi_N(\log N)^3+N^{-2}+\E\Delta.
\]
Since $\phi_N\geq N^{-1}$ and $\E\Delta\lesssim\phi_N$, this proves
\eqref{eq:rate}. Taking $q=0$, $t_0=d$, and $\beta_0=\beta$ proves
\eqref{eq:holderrate}. \qed

\section{Minimax lower bound}\label{app:lower}
We give both submodels explicitly. The independent-measurement argument is an Assouad construction, and the unit term follows from two-point testing \citep{tsybakov2009}.

\subsection{Independent-measurement submodel}
Write $N=nm$. Set $U_i=0$ and use conditional Bernoulli, Poisson, or Gaussian laws with canonical parameter $f(x)$. Then all $N$ measurements are independent, regardless of their recorded unit labels.

Choose a nonzero nonnegative function $\psi\in C_c^\infty((1/4,3/4)^d)$.
Let $K=\lceil N^{1/(2\beta+d)}\rceil$, $h=1/K$, and partition $[0,1]^d$ into $M=K^d$ cubes $Q_1,\ldots,Q_M$ of side $h$. If $a_k$ is the lower corner of $Q_k$, define
\[
 \psi_k(x)=h^\beta\psi\{(x-a_k)/h\}.
\]
The functions are extended by zero off their cubes. For $\omega\in\{0,1\}^M$ and a fixed small $\gamma>0$, put
\begin{equation}\label{eq:bumps}
 f_\omega(x)=\gamma\sum_{k=1}^M\omega_k\psi_k(x),
 \qquad \mu_\omega(x)=b'\{f_\omega(x)\}.
\end{equation}
The supports are disjoint. A derivative of order $j\leq k_\beta$, where
$\beta=k_\beta+\alpha$, is bounded by $C\gamma h^{\beta-j}$.
The $\alpha$-H\"older seminorm of an order-$k_\beta$ derivative on a single support is at most $C\gamma$ by scaling.
For different supports, their separation is at least $h/2$, and their derivative magnitudes are at most $C\gamma h^\alpha$, giving the same seminorm bound.
The smooth zero extension handles a point outside a support. Thus the H\"older norm of $f_\omega$ is bounded by $C\gamma$, uniformly in $K$ and $\omega$.

Choose $\gamma$ small enough that this norm is at most $R$ and
$\norm{f_\omega}_\infty\leq F$. The conditional means then remain in the prescribed interior intervals: $b'(0)=1/2$, $1$, or $0$ in the Bernoulli, Poisson, and Gaussian cases, respectively. All these alternatives lie in $\mathfrak P_{\beta,d}$.

Let $\omega^{(k)}$ be $\omega$ with bit $k$ flipped, and let $P_\omega$ denote the full data law. For a canonical exponential family with cumulant $b$, the conditional KL divergence from parameter $t$ to parameter $s$ is
\[
 b(s)-b(t)-b'(t)(s-t)\leq\frac{\kappa_+}{2}(s-t)^2.
\]
The design law is identical under all alternatives. Additivity of KL for independent measurements gives
\begin{align}
 \KL(P_\omega,P_{\omega^{(k)}})
 &\leq \frac{\kappa_+}{2}N\int(f_\omega-f_{\omega^{(k)}})^2\,dx\notag\\
 &=\frac{\kappa_+}{2}N\gamma^2h^{2\beta+d}\int\psi^2\,dx
 \leq C\gamma^2,\label{eq:bumpkl}
\end{align}
because $Nh^{2\beta+d}\leq1$. Reducing $\gamma$ once more ensures that
\eqref{eq:bumpkl} is at most $1/8$.
Pinsker's inequality implies
\begin{equation}\label{eq:tvbump}
 \TV(P_\omega,P_{\omega^{(k)}})\leq1/4.
\end{equation}

On cube $Q_k$, the mean function takes one of two forms:
$\mu_{k,0}(x)=b'(0)$ and
$\mu_{k,1}(x)=b'\{\gamma\psi_k(x)\}$. Their squared separation is at least
\begin{equation}\label{eq:separation}
 d_k^2=\int_{Q_k}(\mu_{k,1}-\mu_{k,0})^2\,dx
 \geq\kappa_-^2\gamma^2h^{2\beta+d}\int\psi^2\,dx.
\end{equation}
For any candidate estimator $\widetilde\mu$, define $\widehat\omega_k$ to be the bit whose mean function is closer to $\widetilde\mu$ in $L_2(Q_k)$, resolving ties deterministically.
If $\widehat\omega_k\ne\omega_k$, the triangle inequality and the definition of the closer mean give
\[
 \int_{Q_k}(\widetilde\mu-\mu_\omega)^2\,dx\geq d_k^2/4.
\]
Thus
\begin{equation}\label{eq:decode}
 2^{-M}\sum_\omega \E_\omega\norm{\widetilde\mu-\mu_\omega}_{L_2(dx)}^2
 \geq\frac14\sum_{k=1}^M d_k^2\,
       2^{-M}\sum_\omega\PP_\omega(\widehat\omega_k\ne\omega_k).
\end{equation}
Pairing each $\omega$ with $\omega^{(k)}$, the sum of the two testing error probabilities is at least $1-\TV(P_\omega,P_{\omega^{(k)}})$.
The average error in \eqref{eq:decode} is therefore at least $3/8$.
Using \eqref{eq:separation} and $Mh^d=1$ yields
\[
 \sup_\omega\E_\omega\norm{\widetilde\mu-\mu_\omega}_{L_2(dx)}^2
 \geq c h^{2\beta}.
\]
Since $K\leq2N^{1/(2\beta+d)}$ for $N\geq1$, this is at least
$c'N^{-2\beta/(2\beta+d)}$. The bound holds for every estimator.

\subsection{Random-intercept submodel}
Let $\mu_c=b'(0)$. Choose a fixed $a>0$ so small that $\mu_c-a$ and $\mu_c+a$ are allowed conditional means, $2a$ is below the uniform random-effect bound, and $a^2$ is below its integrated variance bound.
For $|\eta|\leq1/2$, let
\[
 \PP_\eta(V_i=1)=\frac{1+\eta}{2},\qquad
 \PP_\eta(V_i=-1)=\frac{1-\eta}{2},\qquad
 M_i(x)=\mu_c+aV_i.
\]
Then
\[
 \mu_\eta(x)=\mu_c+a\eta,\qquad
 U_i(x)=a(V_i-\eta),\qquad
 \E_\eta\pnorm{U_i}^2=a^2(1-\eta^2).
\]
The marginal index $f_\eta=(b')^{-1}(\mu_c+a\eta)$ is constant.
For sufficiently small fixed $a$ it is in the specified H\"older ball and in $[-F,F]$.
All constants are uniform in $\eta$ and $m$.

Take $\eta_\pm=\pm d_0/\sqrt n$, with a fixed $0<d_0\leq1/8$.
The supports of $M_i$ are the same at the two alternatives.
Given $M_1,\ldots,M_n$, the kernel generating the design variables and the responses in \eqref{eq:model} does not depend on $\eta$.
The data-processing inequality and independence of the latent variables imply
\begin{align}
 \KL(P_{\eta_+},P_{\eta_-})
 &\leq n\KL\left({\rm Bern}\!\left(\frac{1+\eta_+}{2}\right),
                 {\rm Bern}\!\left(\frac{1-\eta_+}{2}\right)\right)\notag\\
 &=n\eta_+\log\frac{1+\eta_+}{1-\eta_+}
 \leq4n\eta_+^2=4d_0^2\leq1/16.\label{eq:latentkl}
\end{align}
For $0\leq t\leq1/2$, the logarithm bound in the last line follows by integrating
$2/(1-t^2)\leq 8/3<4$. In particular, the total variation distance between the two data laws is bounded above by a fixed constant smaller than $1/2$, uniformly in $m$.

For an arbitrary function estimator with finite integrated risk, let
$\widehat t=\int\widetilde\mu(x)\,dx$.
Jensen's inequality gives
\[
 \int\{\widetilde\mu(x)-\mu_\eta\}^2\,dx\geq(\widehat t-\mu_\eta)^2.
\]
The two scalar means differ by
$d_\mu=2ad_0/\sqrt n$.
Classify $\widehat t$ by the nearer mean. Under a misclassification,
$(\widehat t-\mu_\eta)^2\geq d_\mu^2/4$.
The sum of testing errors is at least $1-\TV(P_{\eta_+},P_{\eta_-})$, so
\begin{align*}
 \max_{\eta\in\{\eta_+,\eta_-\}}
 \E_\eta\pnorm{\widetilde\mu-\mu_\eta}^2
 &\geq\frac{d_\mu^2}{8}
        \{1-\TV(P_{\eta_+},P_{\eta_-})\}\\
 &\geq c/n .
\end{align*}
Estimators with infinite risk already satisfy the lower bound.

\subsection{Completion of Theorem~\ref{thm:lower}}
Both submodels lie in the same fixed class $\mathfrak P_{\beta,d}$. The minimax risk over that class is at least the larger of their lower bounds. For nonnegative $a,b$,
$\max(a,b)\geq(a+b)/2$. This proves \eqref{eq:lower} after adjusting the constant.
The positive allowance for random-intercept variation is needed for the $n^{-1}$ term; the subclass with $\tau=0$ uses only the independent-measurement bound. \qed

\section{Proofs for subsampling inference}\label{app:inference}
All expectations in this appendix include training randomness when it is present. Conditional expectations denoted by $\E_*$ hold the observed unit records fixed and average over a new uniform subset and a new training seed. Vector norms are Euclidean; $\norm{\cdot}_F$ denotes the Frobenius matrix norm. We suppress the dependence of kernels and distributions on $n$.

\subsection{Orthogonal decomposition}
We first record the variance identities used in both inference theorems. Let $g(\cO_1,\ldots,\cO_r)$ be a symmetric, square-integrable vector kernel with mean $\vartheta$. For $0\leq k\leq r$, define
\[
 g^{[k]}(o_1,\ldots,o_k)=
 \E g(o_1,\ldots,o_k,\cO_{k+1},\ldots,\cO_r),
 \qquad g^{[0]}=\vartheta.
\]
Its order-$k$ canonical kernel, for $k\geq1$, is
\[
 g_k(o_1,\ldots,o_k)=
 \sum_{A\subseteq\{1,\ldots,k\}}(-1)^{k-|A|}g^{[|A|]}(o_A).
\]
Inclusion--exclusion gives
\begin{equation}\label{eq:kerneldecomp}
 g(\cO_1,\ldots,\cO_r)-\vartheta
 =\sum_{k=1}^r\ \sum_{\substack{A\subseteq\{1,\ldots,r\}\\ |A|=k}}g_k(\cO_A).
\end{equation}
Integrating any one argument of $g_k$ pairs terms with opposite signs and gives zero. Consequently, canonical terms on different index sets are orthogonal: select an index belonging to one set but not the other and condition on all remaining records. This proves orthogonality coordinatewise and after any deterministic linear transformation.

For the complete U-statistic $U_g=\binom{n}{r}^{-1}\sum_{|S|=r}g(\cO_S)$, a fixed set of $k$ records appears in $\binom{n-k}{r-k}$ summands. Therefore
\begin{equation}\label{eq:udecomp}
 U_g-\vartheta
 =\sum_{k=1}^r\frac{\binom{r}{k}}{\binom{n}{k}}
       \sum_{|A|=k}g_k(\cO_A).
\end{equation}
For every fixed matrix $D$ of compatible dimension, orthogonality yields
\begin{align}
 \E\norm{D(g-\vartheta)}^2
 &=\sum_{k=1}^r\binom{r}{k}\E\norm{Dg_k}^2,\label{eq:kernelvariance}\\
 \E\norm{D(U_g-\vartheta)}^2
 &=\sum_{k=1}^r\frac{\binom{r}{k}^2}{\binom{n}{k}}
                       \E\norm{Dg_k}^2.\label{eq:uvariance}
\end{align}
These identities also hold for matrices $D=D_n$ that vary deterministically with $n$. Since
\[
 \frac{\binom{r}{k}}{\binom{n}{k}}
 =\prod_{j=0}^{k-1}\frac{r-j}{n-j}\leq\frac rn,
 \qquad 1\leq k\leq r\leq n,
\]
we obtain the contraction bound
\begin{equation}\label{eq:ucontraction}
 \E\norm{D(U_g-\vartheta)}^2
 \leq\frac rn\E\norm{D(g-\vartheta)}^2.
\end{equation}
For $k\geq2$, the ratio of binomial coefficients is at most $r(r-1)/\{n(n-1)\}$. Subtracting the first projection from \eqref{eq:udecomp} thus gives
\begin{equation}\label{eq:secondcontraction}
 \E\left\|D\left(U_g-\vartheta-\frac rn\sum_i g_1(\cO_i)\right)\right\|^2
 \leq\frac{r(r-1)}{n(n-1)}\E\norm{D(g-\vartheta)}^2.
\end{equation}

\subsection{Proof of Theorem~\ref{thm:projection}}
Let $W_{ni}=\Gamma_r^{-1/2}h_{1,r}(\cO_i)$. Within each row these vectors are independent, have mean zero, and have covariance $I_J$. For a fixed vector $v\in\RR^J$,
\[
 \frac{\sum_{i=1}^n\E|v^{\mathsf T}W_{ni}|^{2+\delta}}
      {n^{1+\delta/2}}
 \leq\norm{v}^{2+\delta}n^{-\delta/2}\E\norm{W_{n1}}^{2+\delta}
 \longrightarrow0.
\]
The scalar Lyapunov theorem, followed by the Cram\'er--Wold argument, gives
\[
 V_n^{-1/2}L_{n,r}=n^{-1/2}\sum_iW_{ni}
 \Longrightarrow N_J(0,I_J).
\]
The two terms in \eqref{eq:projconditions} make the normalized higher-order remainder and bias negligible, the former by Markov's inequality. Slutsky's theorem proves the assertion.

To verify the stated sufficient condition, use \eqref{eq:secondcontraction} with $g=h_r$ and $D=\Gamma_r^{-1/2}$. Since $V_n^{-1/2}=(\sqrt n/r)\Gamma_r^{-1/2}$,
\[
 \E\norm{V_n^{-1/2}R^{(2)}_{n,r}}^2
 \leq\frac{r-1}{r(n-1)}
       \E\norm{\Gamma_r^{-1/2}(h_r-\theta_r)}^2.
\]
This tends to zero under \eqref{eq:projection-sufficient}. When $r=1$, there is no higher-order remainder. \qed

\subsection{Asymptotic distribution under Assumption~\ref{ass:al}}
For this subsection and the next two, apply the deterministic transformation $\Sigma_n^{-1/2}$ to the prediction vectors, their expectations, the influence functions, and the remainders, and suppress tildes. Thus the normalized influence covariance is $I_J$, its fourth moment is $o(n)$, and $r\E\norm{R_r}^2\to0$. The raw IJ matrix, the subtraction $D_B$, and the empirical covariance term transform by congruence. The positive-part operation need not do so; we handle it separately by showing that the matrix to which it is applied is already positive definite with probability tending to one.

Write $\psi_i=\psi_n(\cO_i)$ and $\overline\psi=n^{-1}\sum_i\psi_i$. Taking the seed expectation in \eqref{eq:al} gives
\[
 h_r-\theta_r=\frac1r\sum_{i=1}^r\psi_i+R_r^h,
 \qquad R_r^h=\E_\omega R_r,
 \qquad \E\norm{R_r^h}^2\leq\E\norm{R_r}^2=o(r^{-1}).
\]
The kernel $R_r^h$ is symmetric and centered. Averaging over subsets and applying \eqref{eq:ucontraction} shows that
\begin{equation}\label{eq:ual}
 U_{n,r}-\theta_r=\overline\psi+U_R,
 \qquad \E\norm{U_R}^2\leq\frac rn\E\norm{R_r^h}^2=o(n^{-1}).
\end{equation}
The fourth-moment bound implies the Lyapunov condition for the normalized $\psi_i$. Thus
\begin{equation}\label{eq:ualclt}
 \sqrt n\,(U_{n,r}-\theta_r)
 \Longrightarrow N_J(0,I_J).
\end{equation}
In particular, \eqref{eq:ual} and Cauchy--Schwarz imply
$\norm{n\Cov(U_{n,r})-I_J}_{\rm op}\to0$ in these coordinates.

Let $v_* =\E_*\norm{T_b-U_{n,r}}^2$. From \eqref{eq:al},
\begin{equation}\label{eq:basevariance}
 \E\norm{T_r-\theta_r}^2
 \leq\frac{2J}{r}+2\E\norm{R_r}^2=O(r^{-1}),
 \qquad \E v_*\leq\E\norm{T_r-\theta_r}^2.
\end{equation}
Conditional independence of the fits gives
\begin{equation}\label{eq:mcvariance}
 \E\norm{\overline T-U_{n,r}}^2=\frac{\E v_*}{B}=O((rB)^{-1}).
\end{equation}
Under $B/n\to\infty$ this is $o(n^{-1})$. Combining \eqref{eq:ualclt}, \eqref{eq:mcvariance}, and \eqref{eq:biasal} proves
\begin{equation}\label{eq:ensembleclt}
 \sqrt n\,(\overline T-\mu_{\mathbf x})
 \Longrightarrow N_J(0,I_J).
\end{equation}
The same conditioning argument, in matrix form, proves the exact identity \eqref{eq:totalvariance} for every finite $B$.

\subsection{Ideal infinitesimal jackknife}
Let $Z=(Z_1,\ldots,Z_n)^{\mathsf T}$ denote membership in a new uniform subset of size $r$, put $A=Z-p\mathbf1$, and set
\[
 C_i=\E_*\{(Z_i-p)(T_r-U_{n,r})\},\qquad
 V_{\rm IJ}^{\infty}=a_n^2\sum_i C_iC_i^{\mathsf T}.
\]
The following elementary identities keep track of sampling without replacement:
\begin{equation}\label{eq:membership}
 \Cov_*(Z_i,Z_j)=
 \begin{cases}p(1-p),&i=j,\\-p(1-p)/(n-1),&i\ne j,\end{cases}
 \quad
 \E_*AA^{\mathsf T}=\lambda_n\left(I_n-\frac{\mathbf1\mathbf1^{\mathsf T}}n\right),
 \quad \lambda_n=\frac{r(n-r)}{n(n-1)}.
\end{equation}
They follow from $\E_*Z_i=r/n$ and $\E_*Z_iZ_j=r(r-1)/\{n(n-1)\}$ for $i\ne j$. Also $\norm{A}^2=r(1-p)$ for every subset.

Apply \eqref{eq:al} to the selected records. Decompose $C_i=C_i^{\rm lin}+C_i^R$. The linear term is exactly
\begin{align}
 C_i^{\rm lin}
 &=\frac1r\sum_j\Cov_*(Z_i,Z_j)\psi_j
 =\frac{n-r}{n(n-1)}(\psi_i-\overline\psi),\notag\\
 a_n C_i^{\rm lin}&=\frac{\psi_i-\overline\psi}{n}.\label{eq:ijlinear}
\end{align}
Consequently,
\begin{equation}\label{eq:linearijlimit}
 n a_n^2\sum_i C_i^{\rm lin}(C_i^{\rm lin})^{\mathsf T}
 =\frac1n\sum_i(\psi_i-\overline\psi)(\psi_i-\overline\psi)^{\mathsf T}
 =I_J+o_p(1)
\end{equation}
in operator norm. To justify the last equality for the triangular array, the variance of each entry of $n^{-1}\sum_i\psi_i\psi_i^{\mathsf T}$ is at most $n^{-1}\E\norm{\psi_i}^4=o(1)$, and $\E\norm{\overline\psi}^2=J/n$. The dimension $J$ is fixed.

For the remainder, a useful consequence of \eqref{eq:membership} is
\begin{equation}\label{eq:bessel}
 \sum_i\norm{\E_*(A_i W)}^2
 \leq\lambda_n\E_*\norm{W-\E_*W}^2
\end{equation}
for any square-integrable random vector $W$ defined jointly with the subset and seed. For a scalar $W$, let $d=\E_*(A W)$. We have $\mathbf1^{\mathsf T}d=0$. If $d\ne0$, take $u=d/\norm{d}$ and use conditional Cauchy--Schwarz:
\[
 \norm{d}^2=\{\E_*[(u^{\mathsf T}A)(W-\E_*W)]\}^2
 \leq \E_*(u^{\mathsf T}A)^2\Var_*(W)
 =\lambda_n\Var_*(W).
\]
The case $d=0$ is immediate, and summing this inequality over the coordinates proves \eqref{eq:bessel}.

Use $W=R_r$ on the selected records. An unconditional random subset has the same joint law as $r$ independent records. Hence
\begin{equation}\label{eq:ijresidual}
 \E\sum_i\norm{C_i^R}^2
 \leq\lambda_n\E\norm{R_r}^2=o(n^{-1}).
\end{equation}
The factor $a_n$ is bounded because $r/n\leq\rho<1$. Regard the $C_i^{\rm lin}$ and $C_i^R$ as rows of $n\times J$ matrices. Equation~\eqref{eq:linearijlimit} gives a squared Frobenius norm $O_p(n^{-1})$ for the linear matrix, while \eqref{eq:ijresidual} gives $o_p(n^{-1})$ for the remainder matrix. The inequality
\begin{equation}\label{eq:gramperturb}
 \norm{(C+E)^{\mathsf T}(C+E)-C^{\mathsf T}C}_{\rm op}
 \leq2\norm{C}_F\norm{E}_F+\norm{E}_F^2
\end{equation}
therefore proves
\begin{equation}\label{eq:idealijlimit}
 \norm{nV_{\rm IJ}^{\infty}-I_J}_{\rm op}\to_p0.
\end{equation}

\subsection{Finite ensemble and covariance correction}
Form the $n\times J$ matrix $C$ with rows $C_i^{\mathsf T}$ and its empirical counterpart $\widehat C$. For each fit write $A_b=Z_b-p\mathbf1$, $W_b=T_b-U_{n,r}$, and $\overline A=B^{-1}\sum_b A_b$. Since $\E_*W_b=0$,
\begin{equation}\label{eq:chatdecomp}
 \widehat C-C
 =\left\{\frac1B\sum_b A_b W_b^{\mathsf T}-C\right\}
       -\overline A(\overline T-U_{n,r})^{\mathsf T}.
\end{equation}
The first term on the right is a conditionally centered average. Its conditional mean squared Frobenius norm is at most
\[
 \frac1B\E_*\norm{A_bW_b^{\mathsf T}}_F^2
 =\frac{r(1-p)}B v_*.
\]
For the second term, Jensen's inequality gives the deterministic bound $\norm{\overline A}^2\leq r(1-p)$. By \eqref{eq:mcvariance}, its conditional mean squared norm is at most $r(1-p)v_*/B$ as well. Taking expectations and using \eqref{eq:basevariance} yields
\begin{equation}\label{eq:chaterror}
 \E\norm{\widehat C-C}_F^2=O(B^{-1})=o(n^{-1}).
\end{equation}
From \eqref{eq:idealijlimit}, $\norm{C}_F=O_p(n^{-1/2})$. Applying \eqref{eq:gramperturb} to \eqref{eq:chaterror} proves
\[
 \norm{n\widehat V_{\rm IJ}-I_J}_{\rm op}\to_p0.
\]

It remains to control the two finite-$B$ adjustments. Sample covariance matrices are positive semidefinite. Since centering reduces a sum of squares,
\begin{align*}
 \sum_i\tr\{\widehat{\Cov}_b(Q_{bi})\}
 &\leq\frac1{B-1}\sum_b\sum_i\norm{Q_{bi}}^2\\
 &=\frac{r(1-p)}{B-1}\sum_b\norm{T_b-\overline T}^2.
\end{align*}
Conditionally on the data, the expectation of the last sum is $(B-1)v_*$. Therefore
\begin{equation}\label{eq:correctionbounds}
 \E\tr(D_B)\leq\frac{a_n^2r(1-p)}B\E v_*=O(B^{-1}),
 \qquad
 \E\tr\{B^{-1}\widehat{\Cov}_b(T_b)\}
 =\frac{\E v_*}B=O((rB)^{-1}).
\end{equation}
Their operator norms multiplied by $n$ tend to zero in probability. Hence the normalized matrix $n(\widehat V_{\rm IJ}-D_B)$ converges to $I_J$ and is positive definite with probability tending to one. Positive definiteness is preserved under invertible congruence, so the corresponding matrix in the original coordinates is positive definite on the same event. Its positive part is then the matrix itself. Returning by congruence to normalized coordinates and using \eqref{eq:correctionbounds} proves the second assertion in \eqref{eq:ijlimit}.

\subsection{Studentization and completion of Theorem~\ref{thm:ij}}
In this subsection all symbols again denote the original, untransformed quantities. For a fixed nonzero $v$, the scalar array $v^{\mathsf T}\psi_i/(v^{\mathsf T}\Sigma_n v)^{1/2}$ has mean zero, variance one, and fourth moment $o(n)$. Indeed, it is the projection of $\Sigma_n^{-1/2}\psi_i$ onto the deterministic unit vector $\Sigma_n^{1/2}v/(v^{\mathsf T}\Sigma_n v)^{1/2}$. The scalar Lyapunov theorem applies even when this direction varies with $n$. Equations~\eqref{eq:ual} and \eqref{eq:mcvariance} in normalized coordinates, together with the bias condition, imply the same limit for
\[
 \frac{\sqrt n\,v^{\mathsf T}(\overline T-\mu_{\mathbf x})}
      {(v^{\mathsf T}\Sigma_n v)^{1/2}}.
\]
Equation~\eqref{eq:ijlimit}, applied to the same unit vector, shows that
$n v^{\mathsf T}\widehat Vv/(v^{\mathsf T}\Sigma_n v)\to_p1$, proving the contrast assertion.

For the quadratic statistic put
\[
 Z_n=\Sigma_n^{-1/2}\sqrt n\,(\overline T-\mu_{\mathbf x}),\qquad
 A_n=\Sigma_n^{-1/2}(n\widehat V)\Sigma_n^{-1/2}.
\]
By \eqref{eq:ensembleclt}, $Z_n\Longrightarrow N_J(0,I_J)$; covariance consistency gives $\norm{A_n-I_J}_{\rm op}\to_p0$. With probability tending to one $A_n$ is invertible, and the quadratic statistic is $Z_n^{\mathsf T}A_n^{-1}Z_n$. It converges to $\chi_J^2$ by Slutsky's theorem and continuity. It may be assigned any value on the event of singular estimated covariance, whose probability tends to zero. This completes the proof. \qed

\end{document}